\documentclass[10pt,twocolumn]{article}

\usepackage{amsmath,amssymb,color,graphicx}
\usepackage{relsize}
\usepackage{IEEEtrantools}

\usepackage{algorithm}
\usepackage{algorithmic}

\usepackage[left=0.6in,top=0.7in,right=0.6in,bottom=0.5in,nohead,textheight=10in,footskip=0.3in,columnsep=0.3in]{geometry}

\usepackage{hyperref}

\newenvironment{proof}
    {\noindent{\em Proof.}
    }
    {
\raggedright{$\square$}
\vspace{3pt}
    }

\long\def\ignore#1{}

\def\calF{{\cal F}}

\def\calL{{\cal L}}

\def\calR{{\cal R}}
\def\calS{{\cal S}}

\def\calX{{\cal X}}

\def\IN#1{{I_{#1}}}
\def\OUT#1{{O_{#1}}}
\def\INfw#1{{I^+_{#1}}}
\def\INbw#1{{I^-_{#1}}}
\def\OUTfw#1{{O^+_{#1}}}
\def\OUTbw#1{{O^-_{#1}}}
\def\omegafw{\omega^+}
\def\omegabw{\omega^-}

\def\INp#1{{I'_{#1}}}
\def\OUTp#1{{O'_{#1}}}

\newcommand{\bx}{\mbox{\boldmath $x$}}

\newcommand{\bmu}{\mbox{ $\mu$}}

\newcommand{\bxS}{\mbox{\footnotesize\boldmath $x$}}

\newcommand{\bmuS}{\mbox{\footnotesize $\mu$}}

\def\myparagraph#1{\vspace{0pt}\noindent{\bf #1~~}}

\newtheorem{theorem}{Theorem}

\newtheorem{lemma}[theorem]{Lemma}
\newtheorem{definition}[theorem]{Definition}
\newtheorem{proposition}[theorem]{Proposition}

\begin{document}

\title{A new look at reweighted message passing}
\author{Vladimir~Kolmogorov \\ \tt{vnk@ist.ac.at}
}
\date{}
\maketitle

\begin{figure}
\vspace{-150pt}
\center
\begin{eqnarray*}
\mbox{~~~~~~~~~~~~Copyright IEEE. Accepted to {\em Transactions on Pattern Analysis and Machine Intelligence} (TPAMI)} \\
\mbox{\url{http://ieeexplore.ieee.org/xpl/articleDetails.jsp?arnumber=6926846}~~~~~~~~~~~~~}
\end{eqnarray*}
\vspace{80pt}
\end{figure}

\begin{abstract}
We propose a new family of message passing techniques for MAP estimation in graphical models
which we call {\em Sequential Reweighted Message Passing} (SRMP).
Special cases include well-known techniques such as {\em Min-Sum Diffusion} (MSD)
and a faster {\em Sequential Tree-Reweighted Message Passing} (TRW-S).
Importantly, our derivation is simpler than the original derivation of TRW-S,
and does not involve a decomposition into trees.
This allows easy generalizations. 
The new family of algorithms can be viewed as a generalization of TRW-S from pairwise to higher-order graphical models.
We test SRMP on several real-world problems with promising results.
\end{abstract}



\section{Introduction}
This paper is devoted to the problem of minimizing a function of discrete variables
represented as a sum of {\em factors}, where a factor is a term depending on
a certain subset of variables. The problem is also known as MAP-MRF inference
in a graphical model. Due to the generality of the definition, it has applications in 
many  areas. 
Probably, the most well-studied case is when each factor depends on
at most two variables ({\em pairwise MRFs}). Many inference algorithms have been proposed.
A prominent approach is to try to solve a natural linear programming (LP) relaxation
of the problem, sometimes called {\em Schlesinger LP}~\cite{Werner:PAMI07}.
A lot of research went into developing efficient solvers for this special LP, as detailed below.

One of the proposed  techniques is {\em Min-Sum Diffusion} (MSD)~\cite{Werner:PAMI07}.
It has a very short derivation, but the price for this simplicity is efficiency:
MSD can be significantly slower than more advanced techniques such as {\em Sequential Tree-Reweighted Message Passing} ($\mbox{TRW-S}$)~\cite{Kolmogorov:TRWS:PAMI06}.
The derivation of TRW-S in \cite{Kolmogorov:TRWS:PAMI06} uses additionally a decomposition of the graph into trees (as in~\cite{Wainwright:trw_max}),
namely into {\em monotonic chains}. This makes generalizing TRW-S to other cases harder (compared to MSD).

We consider a simple modification of MSD which we call {\em Anisotropic MSD};
it is equivalent to a special case of the {\em Convex Max-Product} (CMP) algorithm~\cite{Hazan:10}.
We then show that with a particular choice of weights and the order of processing nodes Anisotropic MSD
becomes equivalent to TRW-S (in the case of pairwise graphical models). This gives an alternative derivation of TRW-S that does involve a decomposition into chains,
and allows an almost immediate generalization of TRW-S to higher-order graphical models.

Note that generalized TRW-S has been recently presented in~\cite{GTRWS:arXiv12}.
However, we argue that their generalization is  more complicated:
it introduces more notation and definitions related to a decomposition of the graphical model into {\em monotonic junction chains},
imposes weak assumptions on the graph $(\calF,J)$ (this graph is defined in the next section), uses some restriction on the order of processing factors,
and proposes a special treatment of nested factors to improve efficiency.
All this makes generalized $\mbox{TRW-S}$ in~\cite{GTRWS:arXiv12} more difficult to understand and implement.

We believe that our new derivation may have benefits even for pairwise graphical models.
The family of SRMP algorithms is more flexible compared to $\mbox{TRW-S}$;
as discussed in the conclusions, this may prove to be useful in certain scenarios.

\myparagraph{Related work} Besides~\cite{GTRWS:arXiv12}, the closest related works
are probably \cite{Werner:PAMI10} and \cite{Hazan:10}. The first one presented a generalization of pairwise MSD to higher-order graphical models,
and also described a family of LP relaxations specified by a set of pairs of nested factors
for which the marginalization constraint needs to be enforced. We use this framework in our paper.
The work \cite{Hazan:10} presented a family of {\em Convex Message Passing} algorithms, which we  use as one of our building blocks.

Another popular message passing algorithm is {\em MPLP}~\cite{Globerson:NIPS07,Sontag:thesis10,SonGloJaa_optbook}.
Like MSD, it has a simple formulation (we 
give it in section~\ref{sec:BCascent} alongside with MSD). 
However, our tests indicate that MPLP can be significantly slower than SRMP.

Algorithms discusses so far perform a block-coordinate ascent on the objective function
(and may get stuck in a suboptimal point~\cite{Kolmogorov:TRWS:PAMI06,Werner:PAMI07}).
Many other techniques with similar properties have been proposed, e.g.\ \cite{Kumar:ICML08,SonJaa_aistats09,Meltzer:UAI09,Hazan:10,Zheng:CVPR12,WangKoller:ICML13}.

A lot of research also went into developing algorithms that are guaranteed to converge to an optimal solution of the LP.
Examples include subgradient techniques \cite{Storvik:TIP00,Schlesinger:07,Komodakis:ECCV08,Komodakis:CVPR09},
smoothing the objective with a temperature parameter that gradually goes to zero~\cite{Johnson:07},
proximal projections \cite{Ravikumar:JMLR10}, 
Nesterov schemes \cite{Jojic:ICML10,Savchynskyy:CVPR11},
an augmented Lagrangian method~\cite{Martins:ICML11,Meshi:ECML11},
a proximal gradient method~\cite{Schmidt:EMMCVPR11} (formulated for the general LP in~\cite{Pock-et-al-iccv09}),
a bundle method~\cite{Kappes:CVPR12},
a mirror descent method~\cite{Luong:12},
and a ``smoothed version of TRW-S''~\cite{Savchynskyy:UAI12}. 
%





\section{Background and notation}\label{sec:background}
We closely follow the notation of Werner~\cite{Werner:PAMI10}.
Let $V$ be the set of nodes. For node $v\in V$ let $\calX_v$ be
the finite set of possible labels for $v$. For a subset $\alpha\subseteq V$
let $\calX_\alpha=\otimes_{v\in \alpha}\calX_v$ be the set of labelings of $\alpha$,
and let $\calX=\calX_V$ be the set of labelings of $V$.
Our goal is to minimize the function
\begin{equation}
f(\bx\:|\:\bar\theta)=\sum_{\alpha\in\calF} \bar\theta_\alpha(\bx_\alpha) \, , \quad \bx\in\calX \label{eq:E} 
\end{equation}
where $\calF\subset 2^V$ is a set of non-empty subsets of $V$ (also called {\em factors}),
$\bx_\alpha$ is the restriction of $\bx$ to $\alpha\subseteq V$,
and $\bar\theta$ is a vector with components $(\bar\theta_\alpha(\bx_\alpha)\:|\:\alpha\in\calF,\bx_\alpha\in\calX_\alpha)$.

Let $J$ be a fixed set of pairs of the form $(\alpha,\beta)$ where $\alpha,\beta\in\calF$ and $\beta\subset \alpha$.
Note that $(\calF,J)$ is a directed acyclic graph.
We will be interested in solving the following relaxation of the problem:
\begin{equation}
\min_{\bmuS\in\calL(J)} \;\; \sum_{\alpha\in\calF}\sum_{\bxS_\alpha} \bar\theta_\alpha(\bx_\alpha)\mu_\alpha(\bx_\alpha) 
\label{eq:LPprimal}
\end{equation}
where $\mu_\alpha(\bx_\alpha)\in\mathbb R$ are the variables  and  $\calL(J)$ is the $J$-based {\em local polytope} of $(V,\calF)$:
\begin{equation}
\hspace{1pt}\calL(J)\!=\!\left\{\! \bmu\ge 0 
\begin{picture}(1,0)
  \put(3,-28){\line(0,1){60}}
\end{picture}
 \mbox{\begin{tabular}{l} 
$\;\,\mathlarger\sum\limits_{\bxS_\alpha}\;\mu_\alpha(\bx_\alpha)=1$ \hspace{7pt} $\forall \alpha\in\calF, \bx_\alpha\!\!\!$ \\
$\mathlarger\sum\limits_{\bxS_\alpha:\bxS_{\alpha}\sim\bxS_\beta}\!\!\!\!\!\!\mu_\alpha(\bx_\alpha) =\mu_\beta(\bx_\beta)$ \\ \hspace{66pt}$\forall(\alpha,\beta)\in J,\bx_\beta\!\!\!$ 
\end{tabular}} \right\}\hspace{-5pt}
\label{eq:LocalPolytope}
\end{equation}
%
We use the following implicit restriction convention: for $\beta\subseteq \alpha$, whenever symbols $\bx_\alpha$ and $\bx_\beta$ appear in
a single expression they do not denote independent joint
states but $\bx_\beta$ denotes the restriction of $\bx_\alpha$ to nodes in $\beta$.
Sometimes we will emphasize this fact by writing $\bx_\alpha\sim\bx_\beta$, as in the eq.~\eqref{eq:LocalPolytope}. 

An important case that is frequently used is when $|\alpha|\le 2$  for all $\alpha\in\calF$
(a {\em pairwise} graphical model). We will always assume that in this case 
 $J=\{(\{i,j\},\{i\}),(\{i,j\},\{j\})\:|\:\{i,j\}\in\calF\}$.

When there are higher-order factors, one could define $J=\{(\alpha,\{i\})\:|\:i\in\alpha\in\calF,|\alpha|\ge 2\}$. Graph $(\calF,J)$ is then known as a {\em factor graph},
and the resulting relaxation is sometimes called the {\em Basic LP relaxation} (BLP).
It is known that this relaxation is tight if each term $\bar\theta_A$ is a submodular function~\cite{Werner:PAMI10}.
A larger classes of functions that can be solved with BLP has been recently identified in~\cite{ThapperZivny:FOCS12,Kolmogorov:ICALP13},
who in fact completely characterized classes of {\em  Valued Constraint Satisfaction Problems} 
for which the BLP relaxation is always tight.

For many practical problems, however, the BLP relaxation is not tight; then we can  add extra edges to $J$ to tighten the relaxation.
\ignore{
It is worth mentioning that more edges do not always give tighter relaxation; a simple counterexample is given by following proposition.
\begin{proposition}\cite{GTRWS:arXiv12}\label{prop:closure}
The following two operations do not affect the set $\calL(J)$, and thus
relaxation~\eqref{eq:LPprimal}:
(i) Pick edges $(\alpha,\beta),(\beta,\gamma)\in J$, add $(\alpha,\gamma)$ to $J$. 
(ii) Pick edges $(\alpha,\beta),(\alpha,\gamma)\in J$ with $\gamma\subset\beta$, add $(\beta,\gamma)$ to $J$.
\end{proposition}
Note, in general conditions $\alpha,\beta\in\calF$, $\beta\subseteq \alpha$ do not imply that $(\alpha,\beta)\in J$.
Requiring the latter would be unreasonable; if, for example, $|\alpha|,|\beta|\gg 1$ then
adding edge $(\alpha,\beta)$ to $J$ would lead to a relaxation which is computationally infeasible to solve.
}

\myparagraph{Reparameterization and dual problem} For each $(\alpha,\beta)\in J$ 
let $m_{\alpha\beta}$ be a {\em message} from $\alpha$ to $\beta$; it's a vector with components $m_{\alpha\beta}(\bx_\beta)$ for $\bx_\beta\in\calX_\beta$.
Each message vector $m=(m_{\alpha\beta}\:|\:(\alpha,\beta)\in J)$ defines a new vector $\theta=\bar\theta[m]$ according to
\begin{equation}
\theta_\beta(\bx_\beta)=\bar\theta_\beta(\bx_\beta)+\!\!\!\!\!\!\sum_{(\alpha,\beta)\in\IN{\beta}}\!\!m_{\alpha\beta}(\bx_\beta)-\!\!\!\!\!\!\sum_{(\beta,\gamma)\in \OUT{\beta}}\!\!m_{\beta\gamma}(\bx_\gamma)
\label{eq:rep}
\end{equation}
where $\IN{\beta}$ and $\OUT{\beta}$ denote respectively the  set of incoming and outgoing edges for $\beta$:
\begin{equation}
\IN{\beta}=\{(\alpha,\beta)\in J\}\qquad \OUT{\beta}=\{(\beta,\gamma)\in J\}
\end{equation}
It is easy to check that $\bar\theta$ and $\theta$ define the same objective function, i.e.\ $f(\bx\:|\:\bar\theta)=f(\bx\:|\:\theta)$
for all labelings $\bx\in\calX$. Vector $\theta$ that satisfies such condition is called  a {\em reparameterization} of $\bar\theta$  \cite{Wainwright:trw_max}.

For each vector $\theta$  expression 
$
\Phi(\theta) = \sum\limits_{\alpha\in\calF} \min\limits_{\bxS_\alpha} \theta_\alpha(\bx_\alpha)
$
gives a lower bound on $\min_{\bxS} f(\bx\:|\:\theta)$.
For a vector of messages $m$ let us define $\Phi(m)=\Phi(\theta[m])$; as follows from above, it is a lower bound on energy~\eqref{eq:E}.
To obtain the tightest bound, we need to solve the following problem:
\begin{equation}
\max_m \Phi(m)
\end{equation}
It can be checked that this maximization problem is equivalent to the dual of~\eqref{eq:LPprimal} (see \cite{Werner:PAMI10}).

\section{Block-coordinate ascent}\label{sec:BCascent}

To maximize lower bound $\Phi(m)$, we will use a block-coordinate ascent strategy:
select a subset of edges $J'\subseteq J$ and maximize $\Phi(m)$
over messages $(m_{\alpha\beta}\:|\:(\alpha,\beta)\in\widetilde J)$ while keeping
all other messages fixed. It is not difficult to show that such restricted
maximization problem can be solved efficiently if graph $(\calF,J')$ is a tree (or a forest);
for pairwise graphical models this was shown in~\cite{SonJaa_aistats09}. In this paper we restrict
our attention to two special cases of star-shaped trees:
\begin{list}{}{\setlength{\leftmargin}{14pt}\setlength{\labelsep}{5pt}}
\item[I.~] Take $J'\!=\!\INp{\beta}\subseteq\IN{\beta}$ for a fixed factor $\beta\!\in\!\calF$, i.e.\ (a subset of) incoming edges to $\beta$.
\item[II.] Take $J'\!=\!\OUTp{\alpha}\subseteq\OUT{\alpha}$ for a fixed factor $\alpha\!\in\!\calF$, i.e.\ (a subset of) outgoing edges from $\alpha$.
\end{list}
We will mostly focus on the case when we take {\bf all} incoming or outgoing edges, but
for generality we also allow proper subsets. The two procedures described below
are  special cases of the {\em Convex Max-Product} (CMP) algorithm~\cite{Hazan:10} but formulated
in a different way: the presentation in \cite{Hazan:10} did not use the notion of a reparameterization.

\myparagraph{Case I: Anisotropic MSD} 
Consider factor $\beta\in\calF$ and a non-empty set of incoming edges $\INp\beta\subseteq\IN{\beta}$.
A simple algorithm for maximizing $\Phi(m)$ over messages in $\INp\beta$
is {\em Min-Sum Diffusion} (MSD). 
For pairwise models MSD was discovered by Kovalevsky and Koval in the 70's 
and independently by Flach in the 90's (see~\cite{Werner:PAMI07}).
Werner~\cite{Werner:PAMI10} then generalized it to higher-order relaxations.

\begin{figure*}[!t]
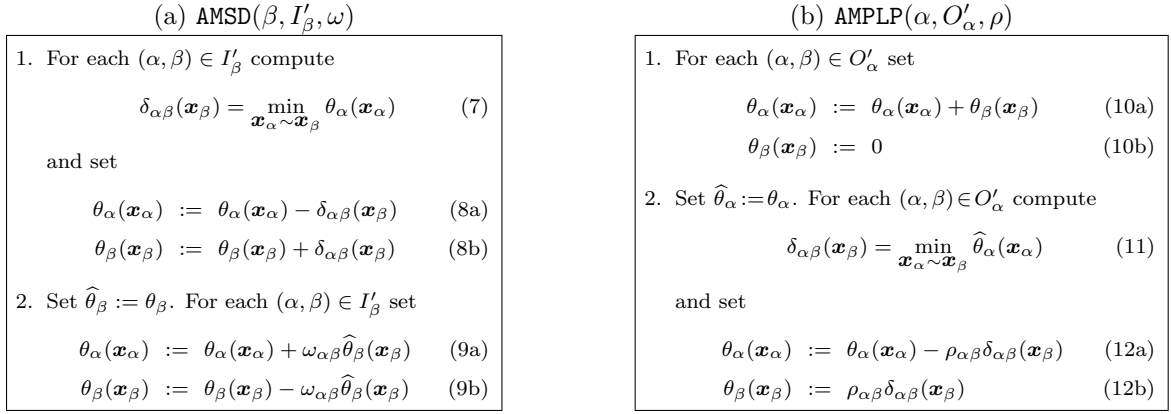

\begin{tabular}{@{\hspace{27pt}}c@{\hspace{50pt}}c}
(a) {\tt AMSD}$(\beta,\INp{\beta},\omega)$ & (b) {\tt AMPLP}$(\alpha,\OUTp{\alpha},\rho)$ \\

\framebox[66mm][l]{ \small
\begin{minipage}[l]{62mm}
\vspace{0pt}
\footnotesize 
\baselineskip 12pt

\begin{list}{}{\setlength{\leftmargin}{9pt}\setlength{\labelsep}{5pt}}
\item[1.] For each $(\alpha,\beta)\in \INp{\beta}$ compute 
\begin{equation}
\delta_{\alpha\beta}(\bx_\beta)=\min\limits_{\bxS_\alpha\sim\bxS_\beta} \theta_\alpha(\bx_\alpha)\label{eq:AMSDdelta}
\end{equation}
and set
\begin{subequations}
\begin{eqnarray}
\theta_\alpha(\bx_\alpha) &\!\!\!:=\!\!\!& \theta_\alpha(\bx_\alpha) - \delta_{\alpha\beta}(\bx_\beta) \qquad \\
\theta_\beta(\bx_\beta)   &\!\!\!:=\!\!\!& \theta_\beta(\bx_\beta) + \delta_{\alpha\beta}(\bx_\beta)
\end{eqnarray}
\end{subequations}
\item[2.] Set $\widehat\theta_\beta:=\theta_\beta$. For each $(\alpha,\beta)\in \INp{\beta}$ set 
\begin{subequations}
\begin{eqnarray}
\theta_\alpha(\bx_\alpha) &\!\!\!:=\!\!\!& \theta_\alpha(\bx_\alpha) + \omega_{\alpha\beta}\widehat\theta_{\beta}(\bx_\beta) \qquad \\
\theta_\beta(\bx_\beta) &\!\!\!:=\!\!\!& \theta_\beta(\bx_\beta) - \omega_{\alpha\beta}\widehat\theta_{\beta}(\bx_\beta)
\end{eqnarray}
\end{subequations}
\end{list}

\end{minipage}
} 
&
\framebox[71mm][l]{ \small
\begin{minipage}[l]{67mm}
\vspace{0pt}
\footnotesize 
\baselineskip 12pt

\begin{list}{}{\setlength{\leftmargin}{9pt}\setlength{\labelsep}{5pt}}
\item[1.] For each $(\alpha,\beta)\in \OUTp{\alpha}$ set 
\begin{subequations}
\begin{eqnarray}
\theta_\alpha(\bx_\alpha) &\!\!\!:=\!\!\!& \theta_\alpha(\bx_\alpha) + \theta_{\beta}(\bx_\beta) \qquad \\
\theta_\beta(\bx_\beta)   &\!\!\!:=\!\!\!& 0
\end{eqnarray}
\end{subequations}
\item[2.] Set $\widehat\theta_\alpha\!:=\!\theta_\alpha$. For each $(\alpha,\beta)\!\in\! \OUTp{\alpha}$ compute
\begin{equation}
\delta_{\alpha\beta}(\bx_\beta)=\min\limits_{\bxS_\alpha\sim\bxS_\beta} \widehat\theta_\alpha(\bx_\alpha)
\end{equation}
and set
\begin{subequations}
\begin{eqnarray}
\theta_\alpha(\bx_\alpha) & \!\!\!:=\!\!\! & \theta_\alpha(\bx_\alpha) - \rho_{\alpha\beta}\delta_{\alpha\beta}(\bx_\beta)\qquad \\
\theta_\beta(\bx_\beta) & \!\!\!:=\!\!\! &  \rho_{\alpha\beta}\delta_{\alpha\beta}(\bx_\beta)
\end{eqnarray}
\end{subequations}
\end{list}

\end{minipage}
} 
\end{tabular}
\caption[]{\it
{\bf Anisotropic MSD and MPLP updates for factors $\beta$ and $\alpha$ respectively.}
$\omega$ is a probability distribution over $\INp{\beta}\cup\{\beta\}$,
and $\rho$ is a probability distribution over $\OUTp{\alpha}\cup\{\alpha\}$.
All updates should be done for all possible $\bx_\alpha$, $\bx_\beta$ with $\bx_\alpha\sim\bx_\beta$.
}
\label{fig:updates}
\end{figure*}

We will consider a generalization of this algorithm which we call {\em Anisotropic MSD} (AMSD).
It is given in Fig.~\ref{fig:updates}(a). 
In the first step it computes marginals for parents $\alpha$ of $\beta$
and ``moves'' them to factor $\beta$; we call it a {\em collection} step.
It then ``propagates'' obtained vector $\theta_\beta$ back to the parents 
with weights $\omega_{\alpha\beta}$. Here $\omega$ is some probability distribution over $\INp\beta\cup\{\beta\}$,
i.e.\ a non-negative vector with $\sum_{(\alpha,\beta)\in \INp\beta}\omega_{\alpha\beta}+\omega_\beta=1$.
If $\omega_\beta=0$ then vector $\theta_\beta$ will become zero, otherwise some ``mass'' will be kept at $\beta$
(namely, $\omega_\beta \,\widehat\theta_\beta$). 
\footnote{
Alternatively, update {\tt AMSD}$(\beta,\INp\beta,\omega)$
can be defined as follows:
reparameterize vectors $\theta_\beta$ and $\{\theta_\alpha\:|\:(\alpha,\beta)\in\INp\beta\}$
to get 
\begin{eqnarray*}
\theta_\beta(\bx_\beta)&=&\omega_{\beta}\widehat\theta_\beta(\bx_\beta) \\
\min\limits_{\bxS_\alpha\sim\bxS_\beta}\theta_\alpha(\bx_\alpha)&=&\omega_{\alpha\beta}\widehat\theta_\beta(\bx_\beta) \qquad\forall (\alpha,\beta)\in\INp\beta
\end{eqnarray*}
for some vector $\widehat\theta_\beta$.
}
The following fact can easily be shown (see Appendix~\ref{app:MSDproof}).
\begin{proposition}
Procedure {\tt AMSD}$(\beta,\INp\beta,\omega)$ maximizes $\Phi(m)$ over $(m_{\alpha\beta}\:|\:(\alpha,\beta)\in \INp\beta)$.
\label{prop:AMSD}
\end{proposition}

If $\INp\beta$ contains a single edge $\{(\alpha,\beta)\}$ and $\omega$ is a uniform distribution over $\INp\beta\cup\{\beta\}$
(i.e.\ $\omega_{\alpha\beta}=\omega_\beta=\frac{1}{2}$) then the procedure becomes equivalent to MSD.
\footnote{MSD algorithm given in~\cite{Werner:PAMI10} updates just
a single message $m_{\alpha\beta}$ for some edge $(\alpha,\beta)\in J$; this corresponds to {\tt AMSD}$(\beta,\INp{\beta},\omega)$
with $|\INp{\beta}|=1$. If $\INp{\beta}=\IN\beta$ then
{\tt AMSD}$(\beta,\INp{\beta},\omega)$ with the uniform distribution $\omega$ is equivalent to performing single-edge MSD updates until convergence.
Such version of MSD for pairwise energies was mentioned in~\cite[Remark 4]{Werner:PAMI07}, although without an explicit formula.
}
If $\INp\beta=\IN\beta$ then the procedure becomes equivalent to the version of CMP described in~\cite[Algorithm 5]{Hazan:10}
(for the case when $(\calF,J)$ is a factor graph; we need to take $\omega_{\alpha\beta}=c_\alpha/\hat c_{\beta}$).

The work \cite{Hazan:10} used a fixed distribution $\omega$ for each factor.
We will show, however, that allowing non-fixed distributions 
may lead to significant gains in the performance. As we will see in section~\ref{sec:SRMP}, a particular scheme 
together with a particular order of processing factors will correspond
to the TRW-S algorithm~\cite{Kolmogorov:TRWS:PAMI06} (in the case of pairwise models), which is often faster than MSD/CMP.

\myparagraph{Case II: Anisotropic MPLP} 
Let us now consider factor $\alpha\in\calF$ and a non-empty set of outgoing edges $\OUTp\alpha\subseteq \OUT{\alpha}$.
This case can be tackled using the {\em MPLP} algorithm \cite{Globerson:NIPS07,SonGloJaa_optbook}.


Analogously to Case I, it can be generalized to {\em Anisotropic MPLP} (AMPLP) - see Figure~\ref{fig:updates}(b).
In the first step (``collection'') vectors $\theta_\beta$ for children $\beta$ of $\alpha$ 
are ``moved'' to factor $\alpha$.
We then compute min-marginals of $\alpha$ and ``propagate'' them to children $\beta$
with weights $\rho_{\alpha\beta}$. Here $\rho$ is some probability distribution over $\OUTp\alpha\cup\{\alpha\}$.
The following can easily be shown (see Appendix~\ref{app:MPLPproof}).
\begin{proposition}
Procedure {\tt AMPLP}$(\beta,\OUTp\alpha,\rho)$ maximizes $\Phi(m)$ over $(m_{\alpha\beta}\:|\:(\alpha,\beta)\in \OUTp\alpha)$.
\label{prop:AMPLP}
\end{proposition}

The updates given in \cite{Globerson:NIPS07,Sontag:thesis10,SonGloJaa_optbook}
correspond to {\tt AMPLP}$(\beta,\OUT{\alpha},\rho)$ where $\rho$ is a uniform probability distribution over $\OUT{\alpha}$
(with $\rho_\alpha=0$). By analogy with Case I, we conjecture that a different weighting (that depends on the order of processing factors)
could lead to faster convergence. However, we leave this as a question for future research, and focus on Case I instead.

For completeness, in Appendix~\ref{app:MPLPmessages} we give an implementation of AMPLP via messages;
it is slightly different from implementations in \cite{Globerson:NIPS07,Sontag:thesis10,SonGloJaa_optbook}
since we store explicitly vectors $\theta_\beta$ for factors $\beta$ that have at least one incoming edge $(\alpha,\beta)\in J$.


\section{Sequential Reweighted Message Passing}\label{sec:SRMP}

In this section we consider a special case of anisotropic MSD updates which we call a {\em Sequential Reweighted Message Passing} (SRMP).
To simplify the presentation, we will assume that $|\OUT{\alpha}|\ne 1$ for all $\alpha\in\calF$.
(This is not a severe restriction: if there is factor $\alpha$ with a single child $\beta$ then we can reparameterize $\theta$
to get $\min\limits_{\bxS_\alpha\sim\bxS_\beta}\theta_\alpha(\bx_\alpha)=0$ for all $\bx_\beta$, and then remove factor $\alpha$; this will not affect the relaxation.)

Let $\calS\subset \calF$ be the set of factors that have at least one incoming edge. 
Let us select some total order $\preceq$ on $\calS$.
SRMP will alternate between a forward pass (processing factors $\beta\in\calS$ in the order $\preceq$)
and a backward pass (processing these factors in the reverse order),
with $\INp{\beta}=\IN\beta$.
Next, we discuss how we select distributions $\omega$ over $\IN{\beta}\cup\{\beta\}$ for $\beta\in\calS$.
We will use different distributions during forward and backward passes; they will be denoted as $\omegafw$ and $\omegabw$ respectively.

Let $\INfw{\beta}$ be the set of edges $(\alpha,\beta)\in \IN{\beta}$
such that $\alpha$ is accessed after calling {\tt AMSD}$(\beta,\IN{\beta},\omegafw)$ in the forward pass and before calling
{\tt AMSD}$(\beta,\IN{\beta},\omegabw)$ in the subsequent backward pass. 
Formally,
\begin{subequations}\label{eq:weightsTRWS:ab}
\begin{equation}
\hspace{1pt}\INfw{\beta}=\left\{ (\alpha,\beta)\in J \;
\begin{picture}(1,0)
  \put(3,-10){\line(0,1){24}}
\end{picture}
\; \mbox{\begin{tabular}{l} 
$(\alpha\in\calS\mbox{ AND }\alpha\succ\beta)\mbox{ ~OR}$ \\
$(\exists (\alpha,\gamma)\in J\mbox{ s.t.\ }\gamma\succ\beta)$ 
\end{tabular}} \right\}\hspace{-10pt}
\label{eq:weightsTRWS:a}
\end{equation}
%
Similarly, let $\OUTfw{\beta}$ be the set of edges $(\beta,\gamma)\in J$
such as $\gamma$ is processed after $\beta$ in the forward pass:
\begin{eqnarray}
\OUTfw{\beta}&=&\{(\beta,\gamma)\in J\:|\:\gamma\succ \beta\} 
\label{eq:weightsTRWS:b}
\end{eqnarray}
(Note that $\gamma\in\calS$ since $\gamma$ has an incoming edge,
so the comparison $\gamma\succ \beta$ is valid.)
\end{subequations}
We propose the following formula as the default weighting for SRMP in the forward pass:
\begin{eqnarray}
\hspace{-5pt}\omegafw_{\alpha\beta} &\!\!\!\!\!=\!\!\!\!\!& \begin{cases}
\frac{1}{|\OUTfw{\beta}|+\max\{|\INfw{\beta}|,|\IN{\beta}-\INfw{\beta}|\}} & \mbox{if }(\alpha,\beta)\!\in\! \INfw{\beta} \\
0 & \mbox{if }(\alpha,\beta)\!\in\! \IN{\beta}-\INfw{\beta}\;\;\;\;
\end{cases}
\label{eq:weightsTRWS}
\end{eqnarray}
It can be checked that the weight $\omegafw_\beta=1-\sum_{(\alpha,\beta)\in \IN{\beta}} \omegafw_{\alpha\beta}$ is non-negative, so this is a valid weighting.
We define sets $\INbw{\beta}$, $\OUTbw{\beta}$ and weights $\omegabw_{\alpha\beta}$ for the backward pass in a similar way;
the only difference to \eqref{eq:weightsTRWS:ab}, \eqref{eq:weightsTRWS} is that ``$\succ$'' is replaced with ``$\prec$'':
%
\begin{subequations}
\begin{equation}
\hspace{1pt}\INbw{\beta}=\left\{ (\alpha,\beta)\in J \;
\begin{picture}(1,0)
  \put(3,-10){\line(0,1){24}}
\end{picture}
\; \mbox{\begin{tabular}{l} 
$(\alpha\in\calS\mbox{ AND }\alpha\prec\beta)\mbox{ ~OR}$ \\
$(\exists (\alpha,\gamma)\in J\mbox{ s.t.\ }\gamma\prec\beta)$ 
\end{tabular}} \right\}\hspace{-10pt}
\label{eq:weightsTRWS:bw:a}
\end{equation}
\begin{eqnarray}
\OUTbw{\beta}&=&\{(\beta,\gamma)\in J\:|\:\gamma\prec \beta\} 
\label{eq:weightsTRWS:bw:b}
\end{eqnarray}
\end{subequations}
\begin{eqnarray}
\hspace{-5pt}\omegabw_{\alpha\beta} &\!\!\!\!\!=\!\!\!\!\!& \begin{cases}
\frac{1}{|\OUTbw{\beta}|+\max\{|\INbw{\beta}|,|\IN{\beta}-\INbw{\beta}|\}} & \mbox{if }(\alpha,\beta)\!\in\! \INbw{\beta} \\
0 & \mbox{if }(\alpha,\beta)\!\in\! \IN{\beta}-\INbw{\beta}\;\;\;\;
\end{cases}
\label{eq:weightsTRWS:bw}
\end{eqnarray}
Note that $\IN{\beta}=\INfw{\beta}\cup\INbw{\beta}$.
Furthermore, in the case of pairwise models sets 
$\INfw{\beta}$ and $\INbw{\beta}$ are disjoint.

Our motivation for the choice~\eqref{eq:weightsTRWS} is as follows.
First of all, we claim that weight $\omegafw_{\alpha\beta}$ for edges $(\alpha,\beta)\in \IN{\beta}-\INfw{\beta}$ can be set to zero without loss of generality.
Indeed, if $\omegafw_{\alpha\beta}=c>0$ then we can transform
$\omegafw_{\alpha\beta}:=0$, $\omegafw_{\beta}:=\omegafw_{\beta}+c$ without affecting the behaviour of the algorithm.

We also decided to set weights $\omegafw_{\alpha\beta}$ to the same value for all edges $(\alpha,\beta)\in\INfw{\beta}$; let us call it $\lambda$.
We must have $\lambda\le \frac{1}{|\INfw{\beta}|}$ to guarantee that $\omegafw_\beta\ge 0$.

If $\OUTfw{\beta}$ is non-empty then we should leave some ``mass'' at $\beta$, i.e.\ choose $\lambda<\frac{1}{|\INfw{\beta}|}$;
this is needed for ensuring that we get a {\em local arc consistency} 
upon convergence of the lower bound (see section~\ref{sec:Jconsistency}).
This was the reason for adding $|\OUTfw{\beta}|$ to the demoninator of the expression in \eqref{eq:weightsTRWS}.

Expression $\max\{|\INfw{\beta}|,|\IN{\beta}-\INfw{\beta}|\}$ in  \eqref{eq:weightsTRWS} was chosen
to make SRMP equivalent to TRW-S in the case of the pairwise models (this equivalence is discussed later in this section).

\noindent{\bf Remark 1~}
{\em 
Setting
$
\lambda = \frac{1}{1 + |\IN{\beta}|}
$
would give the CMP algorithm (with uniform weights $\omega$) that processes factors in $\calS$ using forward and backward passes.
The resulting weight is usually smaller than the weight in eq.~\eqref{eq:weightsTRWS}.

We conjecture that generalized TRW-S~\cite{GTRWS:arXiv12} is also a special case of SRMP.
In particular, if $J=\{(\alpha,\{i\})\:|\:i\in \alpha\in\calF,|\alpha|\ge 2\}$ then setting
$\lambda=\frac{1}{\max\{|\IN{\beta}-\INbw{\beta}|,|\IN{\beta}-\INfw{\beta}|\}}$ would
 give GTRW-S with a uniform distribution over junction chains (and assuming that we take the longest possible chains).
The resulting weight is the same or smaller than the weight in eq.~\eqref{eq:weightsTRWS}.
}

\noindent{\bf Remark 2~} {\em We tried one other choice, namely setting $\lambda=\frac{1}{|\INfw{\beta}|}$
in the case of pairwise models.
This gives the same or larger weights compared to~\eqref{eq:weightsTRWS}.
Somewhat surprisingly to us, in our preliminary tests this appeared to perform slightly
worse than the choice \eqref{eq:weightsTRWS}. A possible informal explanation is as follows:
operation {\tt AMSD}$(\beta,\IN{\beta},\omegafw)$ sends the mass away
from $\beta$, and it may never come back. It is thus desirable to keep
some mass at $\beta$, especially when $|\INbw{\beta}|\gg|\INfw{\beta}|$.
}

\myparagraph{Implementation via messages}
A standard approach for implementing message passing algorithms is to store original vectors $\bar\theta_\alpha$ for factors $\alpha\in\calF$
 and messages $m_{\alpha\beta}$ for edges $(\alpha,\beta)\in J$
that define current reparameterization $\theta=\bar\theta[m]$ via eq.~\eqref{eq:rep}.
This leads to Algorithm~\ref{alg:SRMP:messages}. As in Fig.~\ref{fig:updates},
all updates should be done for all possible $\bx_\alpha$, $\bx_\beta$ with $\bx_\alpha\sim\bx_\beta$.
As usual, for numerical stability messages can be normalized by an additive constant so that
$\min_{\bxS_\beta}m_{\alpha\beta}(\bx_\beta)=0$ for all $(\alpha,\beta)\in J$;
this does not affect the behaviour of the algorithm.

\begin{algorithm}[!h]
\caption{Sequential Reweighted Message Passing (SRMP). 
}\label{alg:SRMP:messages}
\begin{algorithmic}[1]
\STATE initialization: set $m_{\alpha\beta}(\bx_\beta)=0$ for all $(\alpha,\beta)\in J$
\STATE {\bf repeat} until some stopping criterion
\STATE ~~ {\bf for} each factor $\beta\in\calS$ {\bf do} in the order $\preceq$
\STATE ~~~~~~ for each edge $(\alpha,\beta)\in \INbw{\beta}$  update $m_{\alpha\beta}(\bx_\beta):=$
\begin{equation}
\hspace{0pt}
\min_{\bxS_\alpha\sim\bxS_\beta} 
\left\{ 
    \bar\theta_\alpha(\bx_\alpha)
   +\!\!\!\!\!\!\sum_{(\gamma,\alpha)\in \IN{\alpha}}\!\!\!\!\!\! m_{\gamma\alpha}(\bx_\alpha)
   -\!\!\!\!\!\!\!\!\!\sum_{(\alpha,\gamma)\in \OUT{\alpha},\gamma\ne\beta}\!\!\!\!\!\!\!\!\! m_{\alpha\gamma}(\bx_\gamma)
\right\}\hspace{-10pt}
\label{eq:SRMP:update}
\end{equation}
\STATE ~~~~~~ compute $\theta_\beta(\bx_\beta)=$
$$
\bar\theta_\beta(\bx_\beta)+\!\!\!\!\!\!\sum_{(\alpha,\beta)\in\IN{\beta}}\!\!m_{\alpha\beta}(\bx_\beta)-\!\!\!\!\!\!\sum_{(\beta,\gamma)\in \OUT{\beta}}\!\!m_{\beta\gamma}(\bx_\gamma)
$$
\STATE ~~~~~~ for each edge $(\alpha,\beta)\in \INfw{\beta}$ update \\
  $~~~~~~~~~~~~~m_{\alpha\beta}(\bx_\beta):=m_{\alpha\beta}(\bx_\beta)-\omegafw_{\alpha\beta}\theta_\beta(\bx_\beta)$
\STATE ~~ {\bf end for}
\STATE ~~ reverse ordering $\preceq$, 
swap $\INfw{\beta}\!\leftrightarrow\! \INbw{\beta},\omegafw_{\alpha\beta}\!\leftrightarrow\! \omegabw_{\alpha\beta}$ 
\STATE {\bf end repeat}
\end{algorithmic}
\end{algorithm}

Note that update \eqref{eq:SRMP:update} is performed only for edges $(\alpha,\beta)\in \INbw{\beta}$
while step 1 in {\tt AMSD}$(\beta,\IN{\beta},\omegafw)$ requires updates for edges $(\alpha,\beta)\in \IN{\beta}$.
This discrepancy is justified by the proposition below.

\begin{proposition}
Starting with the second pass, update \eqref{eq:SRMP:update} for edge $(\alpha,\beta)\in \IN{\beta}-\INbw{\beta}$ would not change vector $m_{\alpha\beta}$.
\label{prop:SRMPfwbw}
\end{proposition}
\begin{proof}
We will refer to this update of $\beta$ as ``update II'', and to the previous update of $\beta$ during the preceding backward pass as ``update I''.
For each edge $(\alpha,\gamma)\in J$ with $\gamma\ne \beta$ we have $\gamma\succ\beta$ (since $(\alpha,\beta)\notin\INbw{\beta}$), therefore such $\gamma$ will not be processed
between updates I and II. Similarly, if $\alpha\in\calS$ then $\alpha\succ\beta$ (again, since $(\alpha,\beta)\notin\INbw{\beta}$), so $\alpha$ also will not be
processed. Therefore, vector $\theta_\alpha$ is never modified between updates I and II.
Immediately after update I we have $\min\limits_{\bxS_\alpha\sim\bxS_\beta}\theta_\alpha(\bx_\alpha)=0$
for all $\bx_\beta$, and so vector $\delta_{\alpha\beta}$ in~\eqref{eq:AMSDdelta} during update II would be zero. This implies the claim.
\end{proof}

For pairwise graphical models skipping unnecessary updates reduces the amount of computation by approximately a factor of 2.
Note that the argument of the proposition does not apply to the very first forward pass of Algorithm~\ref{alg:SRMP:messages}.
Therefore, this pass is not equivalent to AMSD updates, and the lower bound may potentially decrease during the first pass.

\myparagraph{Alternative implementation}
At the first glance Algorithm~\ref{alg:SRMP:messages} may appear to be different from the TRW-S algorithm~\cite{Kolmogorov:TRWS:PAMI06} (in the case of pairwise models).
For example, if the graph is a chain then after the first iteration messages in TRW-S
will converge, while in SRMP they will keep changing (in general, they will be different after forward and backward passes).
To show a connection to TRW-S, we will describe an alternative implementation of SRMP with the same update rules as in TRW-S.
We will assume that $|\alpha|\le 2$ for $\alpha\in\calF$ and 
 $J=\{(\{i,j\},\{i\}),(\{i,j\},\{j\})\:|\:\{i,j\}\in\calF\}$.

The idea is to use messages $\widehat m_{(\alpha,\beta)}$ for $(\alpha,\beta)\in J$ that have a different intepretation.
Current reparameterization $\theta$ will be determined from $\bar\theta$ and $\widehat m$ using a two-step procedure:
(i) compute $\widehat\theta=\bar\theta[\widehat m]$ via eq.~\eqref{eq:rep};
(ii) compute
\begin{eqnarray*}
\theta_\alpha(\bx_\alpha)&\!\!\!\!\!=\!\!\!\!\!&\widehat\theta_\alpha(\bx_\alpha)+\!\!\!\!\!\sum_{(\alpha,\beta)\in \OUT{\alpha}}\!\!\!\!\!\omega_{\alpha\beta}\widehat\theta_\beta(\bx_\beta)\;\;\quad\forall \alpha\in\calF, |\alpha|=2 \qquad\label{eq:reptwo:a}\\
\theta_\beta(\bx_\beta)&\!\!\!\!\!=\!\!\!\!\!&\omega_{\beta}\widehat\theta_\beta(\bx_\beta)\hspace{68pt}\;\;\quad\forall \beta\in\calF,|\beta|=1\qquad\label{eq:reptwo:b}
\end{eqnarray*}
where $\omega_{\alpha\beta},\omega_\beta$ are the weights used in the last update for $\beta$.
Update rules with this interpretation are given in Appendix~\ref{app:pairwiseSRMP}; if the weights are chosen as in~\eqref{eq:weightsTRWS} and  \eqref{eq:weightsTRWS:bw}
then these updates are equivalent to those in~\cite{Kolmogorov:TRWS:PAMI06}.

\myparagraph{Extracting primal solution}
We used the following scheme for extracting a primal solution $\bx$.
In the beginning of a forward pass we mark all nodes $i\in V$ as ``unlabeled''.
Now consider procedure {\tt AMSD}$(\beta,\IN{\beta},\omegafw)$
(lines 4-6 in Algorithm~\ref{alg:SRMP:messages}).
We assign labels to all nodes in $i\in\beta$ as follows:
(i) for each $(\alpha,\beta)\in \IN{\beta}$ compute ``restricted'' messages $m^{\tt \star}_{\alpha\beta}(\bx_\beta)$
using the following modification of eq.~\eqref{eq:SRMP:update}:
instead of minimizing over all labelings $\bx_\alpha\sim\bx_\beta$,
we minimize only over those labelings $\bx_\alpha$ that are consistent with currently labeled nodes $i\in\alpha$;
(ii) compute $\theta^\star_\beta(\bx_\beta)=\bar\theta_\beta(\bx_\beta)+\sum_{(\alpha,\beta)\in\IN{\beta}} m^\star_{\alpha\beta}(\bx_\beta)$
for labelings $\bx_\beta$ consistent with currently labeled nodes $i\in\beta$,
and choose a labeling with the smallest cost $\theta^\star_\beta(\bx_\beta)$.
It can be shown that for pairwise graphical models this procedure is equivalent to the one given in~\cite{Kolmogorov:TRWS:PAMI06}.

We use the same procedure in the backward pass. We observed that a forward pass usually
produces the same labeling as the previous forward pass (and similarly for backward passes),
but forward and backward passes often given different results.
Accordingly, we run this extraction procedure every third iteration in both passes,
and keep track of the best solution found so far. (We implemented a similar procedure for MPLP, but it performed worse
than the method in~\cite{MPLPcode} - see Fig.~\ref{fig:plots}(a,g).)

\myparagraph{Order of processing factors} An important question is how to
choose the order $\preceq$ on factors in $\calS$. Assume that nodes in $V$ are totally ordered: $V=\{1,\ldots,n\}$.
We used the following rule for factors $\alpha,\beta\subseteq V$ proposed in~\cite{GTRWS:arXiv12}:
(i) first sort by the minimum node in $\alpha$ and $\beta$;
(ii) if $\min \alpha = \min \beta$ then sort by the maximum node.
For the remaining cases we added some arbitrarily chosen rules.

Thus, the only parameter to SRMP is the order of nodes. The choice of this order is an
important issue which is not addressed in this paper. Note, however, that
in many applications there is a natural order on nodes which often works well.
In all of our tests we processed the nodes in the order they were given.


\section{$J$-consistency and convergence properties}\label{sec:Jconsistency}
It is known that fixed points of the MSD algorithm on graph $(\calF,J)$ are characterized by the {\em local arc consistency
} condition w.r.t.\ $J$, or {\em $\mbox{$J$-consistency}$} for short.
In this section
we show a similar property for SRMP.

We will work with relations $\calR_\alpha\subseteq\calX_\alpha$.
For two relations $\calR_\alpha,\calR_{\alpha'}$ of factors $\alpha,\alpha'$ 
with $\alpha\subset \alpha'$ or $\alpha\supset\alpha'$ we denote
%
\begin{eqnarray}
\pi_{\alpha'}(\calR_\alpha)=\{\bx_{\alpha'}\:|\:\exists \bx_\alpha\in\calR_\alpha\mbox{ s.t. }\bx_{\alpha'}\sim\bx_\alpha\}
\end{eqnarray}
If $\alpha'\subset\alpha$ then $\pi_{\alpha'}(\calR_\alpha)$ is 
usually called a {\em projection} of $\calR_\alpha$ to $\alpha'$.
For a vector $\theta_\alpha=(\theta_\alpha(\bx_\alpha)\:|\:\bx_\alpha\in \calX_\alpha)$ we define relation $\langle\theta_\alpha\rangle\subseteq\calX_\alpha$ via
\begin{eqnarray}
\langle\theta_\alpha\rangle=\arg\min_{\bxS_\alpha} \theta_\alpha(\bx_\alpha)
\end{eqnarray}

\begin{definition} Vector $\theta$ is said to be {$J$-consistent}  
if there exist non-empty relations $(\calR_\beta\subseteq\langle\theta_\beta\rangle\:|\:\beta\in\calF)$ such that
$\pi_\beta (\calR_\alpha ) = \calR_\beta$ for each $(\alpha,\beta)\in J$.
\end{definition}

The main result of this section is the following theorem; it shows that $J$-consistency is a natural stopping criterion for SRMP.

\begin{theorem}
Let $\theta^t=\bar\theta[m^t]$ be the vector produced after $t$ iterations of the SRMP algorithm, with $\theta^0=\bar\theta$.
Then the following holds for $t>0$. \footnote{The condition $t>0$ is added
since updates in the very first iteration of SRMP are not equivalent to AMSD updates, as discussed in the previous section.} \\
(a) If $\theta^t$ is $J$-consistent then $\Phi(\theta^{t'}) \!=\! \Phi(\theta^{t})$ for all $t'\!>\!t$.\\
(b) If $\theta^t$ is not $J$-consistent then $\Phi(\theta^{t'}) > \Phi(\theta^{t})$ for some $t'>t$.\\
(c) If $\theta^\ast$ is a limit point of the sequence $(\theta^t)_t$ then $\theta^\ast$ is $J$-consistent
(and also $\lim_{t\rightarrow\infty}\Phi(\theta^t)=\Phi(\theta^\ast)$).
\label{th:convergence:main}
\end{theorem}

\noindent{\bf Remark 3~}
{\em 
Note that the sequence $(\theta^t)_t$ has at least one limit point $\theta^\ast$ if, for example, vectors $\theta^t$ are bounded.
We conjecture that these vectors always stay bounded, but leave this as an open question.}

\noindent{\bf Remark 4~}
{\em
For other message passing algorithms such as MSD it was conjectured in~\cite{Werner:PAMI07}
that the messages $m^t$ converge to a fixed point $m^\ast$ for $t\rightarrow\infty$.
We would like to emhpasize that this is not the case for the SRMP algorithm, as discussed
in the previous section; in general, the messages would be different after backward and forward passes.
In this respect SRMP differs from other proposed message passing techniques such as MSD, TRW-S and MPLP.
However, a weaker convergence property given in Theorem \ref{th:convergence:main}(c) still holds.
(A similar property has been proven for the pairwise TRW-S algorithm~\cite{Kolmogorov:TRWS:PAMI06},
except that we do not prove that the vectors stay bounded.)
}

The remainder of this section is devoted to the proof of Theorem~\ref{th:convergence:main}.
The proof will be applicable not just to SRMP, but to other sequences of updates ${\tt AMSD}(\beta,\IN\beta,\omega)$
that satisfy certain conditions. The first condition is that the updates consist of the same iteration
that is repeatedly applied infinitely many times, and this iteration visits each factor $\beta\in\calF$ with  $\IN\beta\ne\varnothing$ at
least once. The second condition concerns zero components of distributions $\omega$;
it will hold, in particular, if there are no such components. Details are given below.

\subsection{Proof of Theorem~\ref{th:convergence:main}(a)}
The statement is a special case of the following well-known fact: if $\theta$ is $J$-consistent
then applying any number of tree-structured block-coordinate ascent steps (such as AMSD and AMPLP) will
not increase the lower bound. For completeness, a proof of this fact is given below. 
\begin{proposition}
Suppose that $\theta$ is $J$-consistent with relations $(\calR_\beta\subseteq\langle\theta_\beta\rangle\:|\:\beta\in\calF)$,
and $J'\subseteq J$ is a subset of edges such that graph $(\calF,J')$ is a forest.
Applying a block-coordinate ascent step (w.r.t. $J'$) to $\theta$ preserves $J$-consistency
(with the same relations $(\calR_\beta\:|\:\beta\in\calF)$) and does not change the lower bound $\Phi(\theta)$. 
\label{prop:JisLocalMin}
\end{proposition}
\begin{proof}
We can assume w.l.o.g.\ that $J=J'$: removing edges in $J-J'$ does not affect the claim.

Consider LP relaxation~\eqref{eq:LPprimal}. 
We claim that there exists a feasible vector $\mu$ such that $supp(\mu_\alpha)=\calR_\alpha$
for all $\alpha\in\calF$, where $supp(\mu_\alpha)=\{\bx_\alpha\:|\:\mu_\alpha(\bx_\alpha)>0\}$
is the support of probability distribution $\mu_\alpha$.
Such vector can be constructed as follows. First, for each connected component of $(\calF,J)$
we pick an arbitrary factor $\alpha$ in this component and choose some distribution $\mu_\alpha$
with $supp(\mu_\alpha)=\calR_\alpha$ (e.g.\ a uniform distribution over $\calR_\alpha$).
Then we repeatedly choose an edge $(\alpha,\beta)\in J$
with exactly one ``assigned'' endpoint, and choose a probability distribution
for the other endpoint. Namely, if $\mu_\alpha$ is assigned then set $\mu_\beta$ via
 $\mu_\beta(\bx_\beta)=\sum_{\bxS_\alpha\sim\bxS_\beta} \mu_\alpha(\bx_\alpha)$;
we then have $supp(\mu_\beta)=\pi_\beta(supp(\mu_\alpha))=\pi_\beta(\calR_\alpha)=\calR_\beta$.
If $\mu_\beta$ is assigned then we choose some probability distribution $\hat\mu_\alpha$
with $supp(\hat\mu_\alpha)=\calR_\alpha$, compute its marginal probability
$\hat\mu_\beta(\bx_\beta)=\sum_{\bxS_\alpha\sim\bxS_\beta} \hat\mu_\alpha(\bx_\alpha)$
 and then set $\mu_\alpha(\bx_\alpha)=\frac{\mu_\beta(\bxS_\beta)}{\hat\mu_\beta(\bxS_\beta)}\hat\mu_\alpha(\bx_\alpha)$ for
labelings $\bx_\alpha\in\calR_\alpha$; for other labelings $\mu_\alpha(\bx_\alpha)$ is set to zero.
The fact that $\pi_\beta(\calR_\alpha)=\calR_\beta$ implies that $\mu_\alpha$ is a valid probability distribution with $supp(\mu_\alpha)=\calR_\alpha$. The claim is proved.

Using standard LP duality for~\eqref{eq:LPprimal}, it can be checked that condition $\calR_\alpha\subseteq\langle\theta_\alpha\rangle$
for all $\alpha\in\calF$
is equivalent to the complementary slackness conditions for vectors $\mu$ and $\theta=\bar\theta[m]$ 
(where $\mu$ is the vector constructed above and $m$ is the vector of messages corresponding to $\theta$).
Therefore, $m$ is an optimal dual vector for~\eqref{eq:LPprimal}. This means that applying a block-coordinate ascent
step to $\theta=\bar\theta[m]$ results in a vector $\theta'=\bar\theta[m']$ which is optimal as well: $\Phi(\theta')=\Phi(\theta)$.
The complementary slackness conditions must hold for $\theta'$, so $\calR_\alpha\subseteq\langle\theta'_\alpha\rangle$ 
for all $\alpha\in\calF$.
\end{proof}

\subsection{Proof of Theorem~\ref{th:convergence:main}(b,c)}
Consider a sequence of AMSD updates from Fig.~\ref{fig:updates} where $\INp\beta=\IN{\beta}$.
One difficulty in the analysis
is that some components of distributions $\omega$ may be zeros.
We will need to impose some restrictions on such components. Specifically, we will require the following:
\begin{itemize}
\item[R1] {\em For each call ${\tt AMSD}(\beta,\IN\beta,\omega)$ with $\OUT{\beta}\ne\varnothing$ there holds $\omega_\beta>0$. }
\item[R2] {\em Consider a call ${\tt AMSD}(\beta,\IN\beta,\omega)$ with $(\alpha,\beta)\in\IN\beta$, $\omega_{\alpha\beta}=0$.
This call ``locks'' factor $\alpha$, i.e.\ this factor and its children (except for $\beta$) cannot
be processed anymore until it is ``unlocked'' by calling ${\tt AMSD}(\beta,\IN\beta,\omega')$ with $\omega'_{\alpha\beta}>0$. }
\item[R3] {\em The updates are applied in iterations, where each iteration calls  ${\tt AMSD}(\beta,\IN\beta,\omega)$
for each $\beta\in \calF$ with $\IN\beta\ne\varnothing$ at least once.}
\end{itemize}
Restriction R2 can also be formulated as follows. For each factor $\alpha\in\calF$ let us
keep a variable $\Gamma_\alpha\in\OUT{\alpha}\cup\{\varnothing\}$. In the beginning we set $\Gamma_\alpha:=\varnothing$ for all $\alpha\in\calF$,
and after calling ${\tt AMSD}(\beta,\IN\beta,\omega)$ we update these variables as follows:
for each $(\alpha,\beta)\in\IN\beta$ set $\Gamma_\alpha:=(\alpha,\beta)$ if $\omega_{\alpha\beta}=0$,
and $\Gamma_\alpha:=\varnothing$ otherwise. Condition R2 means that calling ${\tt AMSD}(\beta,\IN\beta,\omega)$
is possible only if (i) $\Gamma_\beta=\varnothing$, and (ii) for each $(\alpha,\beta)\in\IN\beta$
there holds $\Gamma_\alpha\in\{(\alpha,\beta),\varnothing\}$.

It can be seen that the sequence of updates in SRMP (starting with the second pass)
satisfies conditions of the theorem; a proof of this fact is similar to that of Proposition~\ref{prop:SRMPfwbw}.

\begin{theorem}
Consider a sequence of AMSD updates satisfying conditions R1-R3.
If vector $\theta^\circ$ is not $J$-consistent then applying the updates to $\theta^\circ$ will increase the lower bound after at most
$T=1+\sum_{\alpha\in\calF}|\calX_\alpha|+\sum_{(\alpha,\beta)\in J}|\calX_\beta|$ iterations.
\label{th:Jconsistency}
\end{theorem}

Theorem~\ref{th:Jconsistency} immediately implies part (b) of Theorem~\ref{th:convergence:main}.
Using an argument from~\cite{Kolmogorov:TRWS:PAMI06}, we can also prove part (c) as follows. 

Let $(\theta^{t(k)})_k$ be a subsequence of $(\theta^{t})_t$ such
that $\lim_{k\rightarrow\infty}\theta^{t(k)}=\theta^\ast$.
We have
\begin{equation}
\lim_{t\rightarrow\infty}\Phi(\theta^t)=
\lim_{k\rightarrow\infty}\Phi(\theta^{t(k)})=\Phi(\theta^\ast)
\label{eq:GKAUSFHKUAGFKAGAS}
\end{equation}
where the first equality holds since the sequence $(\Phi(\theta^t))_t$ is monotonic,
and the second equality is by continuity of function $\Phi:\Omega\rightarrow\mathbb R$,
where by $\Omega$ we denoted the space of vectors of the form $\theta=\bar\theta[m]$.

We need to show that $\theta^\ast$ is $J$-consistent. Suppose that this is not the case.
Define
mapping $\pi:\Omega\rightarrow\Omega$ as follows: we take vector $\theta\in\Omega$ and apply $T$
iterations of SRMP to it. 
By
 Theorem \ref{th:Jconsistency} we have $\Phi(\pi(\theta^\ast))>\Phi(\theta^\ast)$.
Clearly,  $\pi$ and thus $\Phi\circ\pi$ are continuous mappings, therefore
\begin{equation*}
\lim_{k\rightarrow\infty}\Phi(\pi(\theta^{t(k)}))=\Phi\left(\pi \left(\lim_{k\rightarrow\infty}\theta^{t(k)}\right)\right)=\Phi(\pi(\theta^\ast)) 
\end{equation*}
This implies that $\Phi(\pi(\theta^{t(k)}))>\Phi(\theta^\ast)$ for some index $k$.
Note that $\pi(\theta^{t(k)})=\theta^t$ for $t=T+t(k)$. We obtained that $\Phi(\theta^t)>\Phi(\theta^\ast)$ for some $t>0$;
this contradicts eq.~\eqref{eq:GKAUSFHKUAGFKAGAS} and monotonicity of the sequence $(\Phi(\theta^t))_t$.

We showed that Theorem~\ref{th:Jconsistency} indeed implies Theorem~\ref{th:convergence:main}(b,c).
 It  remains to prove Theorem~\ref{th:Jconsistency}.

\noindent{\bf Remark 5~}
{\em 
Note that while Theorem~\ref{th:convergence:main}(a,b) holds for any sequence of AMSD updates
satisfying conditions R1-R3, we believe that this is not the case for Theorem~\ref{th:convergence:main}(c).
If, for example, the updates for factors $\beta$ used varying distributions $\omega$ whose components $\omega_{\alpha\beta}$
 would tend to zero then the increase of the lower bound could become exponentially smaller with
each iteration,
and $\Phi(\theta^t)$ might not converge to $\Phi(\theta^\ast)$ for a $J$-consistent vector $\theta^\ast$.
In the argument above it was essential that the sequence of updates was repeating,
and the weights $\omega^+_{\alpha\beta}$, $\omega^-_{\alpha\beta}$ used in Algorithm~\ref{alg:SRMP:messages}
were kept constant.
}
\subsection{Proof of Theorem~\ref{th:Jconsistency}}\label{sec:Jconsistency:proof}
The proof will be based on the following fact.
\begin{lemma}
Suppose that operation ${\tt AMSD}(\beta,\IN\beta,\omega)$ does not increase the lower bound.
Let $\theta$ and $\theta'$ be respectively the vector before and after the update, and
$\delta_{\alpha\beta}$, $\widehat\theta_\beta$ be the vectors defined in steps 1 and 2.
Then for any $(\alpha,\beta)\in\IN\beta$ there holds
\begin{subequations}
\begin{IEEEeqnarray}{rCll}
\langle\delta_{\alpha\beta}\rangle&=&\pi_\beta(\langle\theta_\alpha\rangle) & 
   \label{eq:lemmaUpdate:a} \\
\langle\widehat\theta_\beta\rangle&=&\langle\theta_\beta\rangle\cap\bigcap_{(\alpha,\beta)\in\IN\beta}\langle\delta_{\alpha\beta}\rangle 
   \label{eq:lemmaUpdate:b} \\
\langle\theta'_\alpha\rangle&=&\langle\theta_\alpha\rangle\cap\pi_\alpha(\langle\widehat\theta_\beta\rangle)  & \mbox{\!\!\!if }\omega_{\alpha\beta}>0 
   \label{eq:lemmaUpdate:c}\\
\hspace{-20pt}\langle\theta'_\alpha\rangle\cap\pi_\alpha(\langle\widehat\theta_\beta\rangle)&=&\langle\theta_\alpha\rangle\cap\pi_\alpha(\langle\widehat\theta_\beta\rangle)  & \mbox{\!\!\!if }\omega_{\alpha\beta}=0 
   \label{eq:lemmaUpdate:d} 
\end{IEEEeqnarray}
\end{subequations}
\label{lemma:AMSDstep}
\end{lemma}
\begin{proof}
Adding a constant to vectors $\theta_\gamma$ does not affect the claim,
so we can assume w.l.o.g.\ that $\min_{\bxS_\gamma}\theta_\gamma(\bx_\gamma)=0$ for all factors $\gamma$.
This means that $\langle\theta_\gamma\rangle=\{\bx_\gamma\:|\:\theta_\gamma(\bx_\gamma)=0\}$.
We denote $\widehat\theta$ to be the vector after the update in step 1.
By the assumption of this lemma and by Proposition \ref{prop:AMSD} we have $\Phi(\widehat\theta)\le \Phi(\theta')=\Phi(\theta)=0$.

Eq.~\eqref{eq:lemmaUpdate:a} follows directly from the definition of $\delta_{\alpha\beta}$.
Also, we have $\min_{\bxS_\beta}\delta_{\alpha\beta}(\bx_\beta)=\min_{\bxS_\alpha}\widehat\theta_\alpha(\bx_\alpha)=0$ for all $(\alpha,\beta)\in\IN\beta$.

By construction, $\widehat\theta_\beta=\theta_\beta+\sum_{(\alpha,\beta)\in\IN\beta} \delta_{\alpha\beta}$.
All terms of this sum are non-negative vectors, thus vector $\widehat\theta_\beta$ is also non-negative.
We must have $\min_{\bxS_\beta}\widehat\theta_\beta(\bx_\beta)=0$, otherwise we would have $\Phi(\widehat\theta)>\Phi(\theta)$ - a contradiction.
Thus,
\begin{IEEEeqnarray*}{rCl}
\langle\widehat\theta_\beta\rangle&=&\{\bx_\beta\:|\:\widehat\theta_\beta(\bx_\beta)\!=\!0\} \\
&=&\{\bx_\beta\:|\:\theta_\beta(\bx_\beta)\!=\!0\mbox{ \!~AND~\! }\delta_{\alpha\beta}(\bx_\beta)\!=\!0~~\forall(\alpha,\beta)\!\in\!\IN\beta\}
\end{IEEEeqnarray*}
which gives~\eqref{eq:lemmaUpdate:b}.

Consider $(\alpha,\beta)\in\IN\beta$. By construction, $\theta'_\alpha(\bx_\alpha)=\widehat\theta_\alpha(\bx_\alpha)+\omega_{\alpha\beta}\widehat\theta_\beta(\bx_\beta)$. We know that (i) vectors $\theta_\alpha$, $\widehat\theta_\alpha$ and $\widehat\theta_\beta$
are non-negative, and their minina is 0; (ii) there holds
$\bx_\beta\in\langle\widehat\theta_\beta\rangle\subseteq\langle\delta_{\alpha\beta}\rangle$,
so for each $\bx_\alpha$ with $\bx_\beta\in\langle\widehat\theta_\beta\rangle$ we have $\delta_{\alpha\beta}(\bx_\beta)=0$
and therefore $\widehat\theta_\alpha(\bx_\alpha)=\theta_\alpha(\bx_\alpha)$. These facts imply
\eqref{eq:lemmaUpdate:c} and \eqref{eq:lemmaUpdate:d}.
\end{proof}

From now on we assume that applying $T$ iterations to vector $\theta^\circ$ does not increase the lower bound;
we need to show that $\theta^\circ$ is $J$-consistent.

Let us define relations $\calR_\alpha\subseteq\calX_\alpha$ for $\alpha\in\calF$ and $\calR_{\alpha\beta}\in\calX_\beta$
for $(\alpha,\beta)\in J$ using the following procedure.
In the beginning we set $\calR_\alpha=\langle\theta^\circ_\alpha\rangle$ for all $\alpha\in\calF$ and $\calR_{\alpha\beta}=\calX_\beta$ for all $(\alpha,\beta)\in J$.
After calling ${\tt AMSD}(\beta,\IN\beta,\omega)$ we update these relations as follows: 
\begin{itemize}
\item Set $\calR'_\beta:=\langle\widehat\theta_\beta\rangle$ where $\widehat\theta_\beta$ is the vector in step 2 of procedure ${\tt AMSD}(\beta,\IN\beta,\omega)$. \
\item For each $(\alpha,\beta)\in\IN\beta$ set $\calR'_{\alpha\beta}:=\langle\widehat\theta_\beta\rangle$.
If $\omega_{\alpha\beta}>0$ then set $\calR'_\alpha:=\langle\theta'_\alpha\rangle$ where $\theta'_\alpha$ is the vector after the update,
otherwise set $\calR'_\alpha:=\calR_\alpha\cap\pi_\alpha(\calR'_{\alpha\beta})$. 
\end{itemize}
Finally, we update $\calR_\beta:=\calR'_\beta$ and $\calR_{\alpha\beta}:=\calR'_{\alpha\beta}$.


\begin{lemma}
The following invariants are preserved for each $\alpha\in\calF$ during the first $T$ iterations: \\
(a) If $\OUT\alpha\ne\varnothing$ and $\Gamma_\alpha=\varnothing$ then $\calR_\alpha=\langle\theta_\alpha\rangle$. \\
(b) If $\OUT\alpha=\varnothing$ then either $\calR_\alpha=\langle\theta_\alpha\rangle$ or $\theta_\alpha={\bf 0}$. \\
(c) If $\Gamma_\alpha=(\alpha,\beta)$ then $\calR_\alpha=\langle\theta_\alpha\rangle\cap\pi_\alpha(\calR_{\alpha\beta})$. \\
Also, for each $(\alpha,\beta)\in J$ the following is preserved: \\
(d)  $\calR_\beta\subseteq\calR_{\alpha\beta}$. If $\OUT\beta=\varnothing$ then $\calR_\beta=\calR_{\alpha\beta}$. \\
(e)  $\calR_\alpha\subseteq\pi_\alpha(\calR_{\alpha\beta})$. \\
Finally, relations $\calR_\alpha$ and $\calR_{\alpha\beta}$ either shrink or stay the same, i.e.\ they never acquire new elements.
\label{lemma:invariants}
\end{lemma}

\begin{proof}
Checking that properties (a)-(e) hold after initialization is straightforward.
Let us show that a call to procedure ${\tt AMSD}(\beta,\IN\beta,\omega)$ preserves them.
We use the notation as in Lemma~\ref{lemma:AMSDstep}. 
Similarly, we denote $\calR_\gamma,\calR_{\alpha\beta},\Gamma_\alpha$ and $\calR'_\gamma,\calR'_{\alpha\beta},\Gamma'_\alpha$
to be the corresponding quantities before and after the update.


\noindent{\bf Monotonicity of $\calR_{\beta}$ and $\calR_{\alpha\beta}$}.~~ Let us show that
$\calR'_{\beta}=\langle\widehat\theta_\beta\rangle\subseteq\calR_\beta$ and also
 $\calR'_{\alpha\beta}=\langle\widehat\theta_\beta\rangle\subseteq\calR_{\alpha\beta}$
for all $(\alpha,\beta)\in\IN\beta$. 
By parts (a,b) of the induction hypothesis for factor $\beta$ two cases are possible:
\begin{itemize}
\item $\langle\theta_\beta\rangle=\calR_\beta$. Then 
$\langle\widehat\theta_\beta\rangle\subseteq\langle\theta_\beta\rangle=\calR_\beta\subseteq\calR_{\alpha\beta}$
 where the first inclusion is by \eqref{eq:lemmaUpdate:a} and the second inclusion is by the first part of (d).
\item $\OUT\beta=\varnothing$ and $\theta_\beta={\bf 0}$. By the second part of (d) we have $\calR_\beta=\calR_{\alpha\beta}$ for 
all $(\alpha,\beta)\in\IN\beta$, so we just need to show that $\langle\widehat\theta_\beta\rangle\subseteq\calR_\beta$.
There must exist $(\alpha,\beta)\in\IN\beta$ with $\Gamma_\alpha=\varnothing$. For such $(\alpha,\beta)$
we have $\langle\theta_\alpha\rangle=\calR_\alpha\subseteq\pi_\alpha(\calR_{\alpha\beta})$
by (a) and (e). This implies that $\pi_\beta( \langle\theta_\alpha\rangle)\subseteq \calR_{\alpha\beta}$.
Finally, $\langle\widehat\theta_\beta\rangle\subseteq \langle\delta_{\alpha\beta}\rangle=\pi_\beta( \langle\theta_\alpha\rangle)\subseteq \calR_{\alpha\beta}=\calR_{\beta}$ where we used \eqref{eq:lemmaUpdate:b}, \eqref{eq:lemmaUpdate:a} and the second part of (d).
\end{itemize}

\noindent{\bf Monotonicity of $\calR_{\alpha}$ for $(\alpha,\beta)\in\IN\beta$.}~~ Let us show
that $\calR'_\alpha\subseteq\calR_\alpha$. If $\omega_{\alpha\beta}=0$ then $\calR'_\alpha=\calR_\alpha\cap\pi_\alpha(\calR'_{\alpha\beta})$,
so the claim holds. Assume that $\omega_{\alpha\beta}>0$. By~\eqref{eq:lemmaUpdate:c}
we have 
$\calR'_\alpha=\langle\theta'_\alpha\rangle
=\langle\theta_\alpha\rangle\cap\pi_\alpha(\langle\widehat\theta_\beta\rangle)$.
We already showed that $\langle\widehat\theta_\beta\rangle=\calR'_{\alpha\beta}\subseteq\calR_{\alpha\beta}$,
therefore 
$\calR'_\alpha
\subseteq\langle\theta_\alpha\rangle\cap\pi_\alpha(\calR_{\alpha\beta})$.
By (a,c) we have either $\calR_\alpha=\langle\theta_\alpha\rangle$ or 
$\calR_\alpha
=\langle\theta_\alpha\rangle\cap\pi_\alpha(\calR_{\alpha\beta})$;
in each case $\calR'_{\alpha}\subseteq\calR_\alpha$.

To summarize, we showed monotonicity for all relations (relations that are not mentioned above do not change).

\noindent{\bf Invariants (a,b)}.~~ If $\omega_\beta=0$ then $\theta'_\beta={\bf 0}$ (and $\OUT\beta=\varnothing$ due to restriction R2). Otherwise 
$\langle\theta'_\beta\rangle=\langle\omega_\beta\widehat\theta_\beta\rangle=\langle\widehat\theta_\beta\rangle=\calR'_\beta$.
In both cases properties (a,b) hold for factor $\beta$ after the update.

Properties (a,b) also cannot become violated for factors $\alpha$ with $(\alpha,\beta)\in\IN\beta$.
Indeed, (b) does not apply to such factors, and (a) will apply only if $\omega_{\alpha\beta}>0$,
in which case we have $\calR'=\langle\theta'_\alpha\rangle$, as required. We proved that (a,b) are preserved
for all factors.

\noindent{\bf Invariant (c)}.~~ Consider edge $(\alpha,\beta)\in\IN\beta$ with $\Gamma'_\alpha=(\alpha,\beta)$
(i.e.\ with $\omega_{\alpha\beta}=0$). We need to show that $\calR'_\alpha=\langle\theta'_\alpha\rangle\cap\pi_\alpha(\calR'_{\alpha\beta})$.
By construction, $\calR'_{\alpha}=\calR_\alpha\cap\pi_\alpha(\calR'_{\alpha\beta})$,
and by~\eqref{eq:lemmaUpdate:d} we have
have $\langle\theta'_\alpha\rangle\cap\pi_\alpha(\calR'_{\alpha\beta})=\langle\theta_\alpha\rangle\cap\pi_\alpha(\calR'_{\alpha\beta})$
(note that $\calR'_{\alpha\beta}=\langle\widehat\theta_\beta\rangle$).
Thus, we need to show that $\calR_\alpha\cap\pi_\alpha(\calR'_{\alpha\beta})=\langle\theta_\alpha\rangle\cap\pi_\alpha(\calR'_{\alpha\beta})$.
Two cases are possible:
\begin{itemize}
\item $\Gamma_\alpha=\varnothing$. By (a) we have $\calR_\alpha=\langle\theta_\alpha\rangle$, which implies the claim.
\item $\Gamma_\alpha=(\alpha,\beta)$. By (c) we have $\calR_\alpha=\langle\theta_\alpha\rangle\cap\pi_\alpha(\calR_{\alpha\beta})$,
so we need to show that 
$\langle\theta_\alpha\rangle\cap\pi_\alpha(\calR_{\alpha\beta})\cap\pi_\alpha(\calR'_{\alpha\beta})=\langle\theta_\alpha\rangle\cap\pi_\alpha(\calR'_{\alpha\beta})$.
This holds since  $\calR'_{\alpha\beta}\subseteq\calR_{\alpha\beta}$.
\end{itemize}

\noindent{\bf Invariants (d,e)}.~~ These invariants cannot become violated for edges $(\alpha',\beta')\in J$
with $\beta'\ne \beta$ since relation $\calR_{\alpha'\beta'}$ does not change and relations $\calR_{\alpha'},\calR_{\beta'}$ do not grow.
Checking that that (d,e) are preserved for edges $(\alpha,\beta)\in\IN\beta$ is straightforward.
We just discuss one case (all other cases follow directly from the construction).
Suppose that $\omega_{\alpha\beta}>0$. Then
 $\calR'_\alpha=\langle\theta'_\alpha\rangle=\langle\theta_\alpha\rangle\cap\pi_\alpha(\calR'_{\alpha\beta})$ (by~\eqref{eq:lemmaUpdate:c}); this 
implies that $\calR'_\alpha\subseteq\pi_\alpha(\calR'_{\alpha\beta})$, as desired.
\end{proof}

We are now ready to prove Theorem \ref{th:Jconsistency}, i.e.\ that vector $\theta^\circ$ is $J$-consistent.
As we showed, all relations never grow, so after fewer than $T$ iterations 
we will encounter an iteration during which the relations do not change.
Let $(\calR_\alpha\:|\:\alpha\in\calF)$ and $(\calR_{\alpha\beta}\:|\:(\alpha,\beta)\in J)$ be the relations during this iteration.
There holds $\calR_\alpha\subseteq\langle\theta^\circ_\alpha\rangle$ for all $\alpha\in\calF$ (since we had equalities after initialization
and then the relations have either shrunk or stayed the same). 
Consider edge $(\alpha,\beta)\in J$.
At some point during the iteration we call {\tt AMSD}$(\beta,\IN\beta,\omega)$. 
By analyzing this call we conclude that $\calR_{\alpha\beta}=\calR_\beta$. 
From~\eqref{eq:lemmaUpdate:a},~\eqref{eq:lemmaUpdate:b} and Lemma~\ref{lemma:invariants}(a) we get $\calR_\beta=\langle\widehat\theta_\beta\rangle\subseteq \langle\delta_{\alpha\beta}\rangle=\pi_\beta(\langle\theta_\alpha\rangle)=\pi_\beta(\calR_\alpha)$. From Lemma~\ref{lemma:invariants}(e) we obtain that $\calR_\alpha\subseteq\pi_\alpha(\calR_\beta)$.

We showed that $\calR_\beta\subseteq\pi_\beta(\calR_\alpha)$ and $\calR_\alpha\subseteq\pi_\alpha(\calR_\beta)$; this implies that $\calR_\beta=\pi_\beta(\calR_\alpha)$.
Finally, relation $\calR_\beta$ (and thus $\calR_\alpha$) is non-empty since $\calR_\beta=\langle\widehat\theta_\beta\rangle$.

\section{Experimental results}\label{sec:experiments}

\begin{figure*}[!t]
\footnotesize
\begin{tabular}{@{\hspace{0pt}}c}\vspace{-12pt} \\
\includegraphics[scale=0.4,trim=  0pt 136pt 300pt 0pt,clip=true]{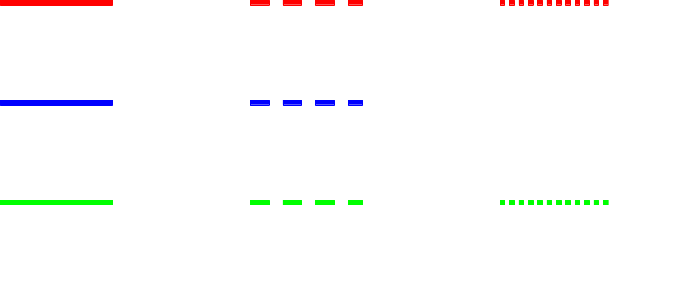} \vspace{-7pt} \\
\includegraphics[scale=0.4,trim=  0pt  89pt 300pt 47pt,clip=true]{legend.pdf} \vspace{-6pt} \\
\includegraphics[scale=0.4,trim=  0pt  40pt 300pt 96pt,clip=true]{legend.pdf} \vspace{0pt}
\end{tabular}
\begin{tabular}{@{\hspace{-5pt}}c}
lower bound. {\color{red}Red=SRMP},
{\color{blue}blue=CMP},
{\color{green}green=MPLP}.~~
{\color{black}Black in (b,c,e,f)=GTRW-S.}
\end{tabular}
\\
\begin{tabular}{@{\hspace{0pt}}c}\vspace{-12pt} \\
\includegraphics[scale=0.4,trim=120pt 136pt 180pt 0pt,clip=true]{legend.pdf} \vspace{-7pt} \\
\includegraphics[scale=0.4,trim=120pt  89pt 180pt 47pt,clip=true]{legend.pdf} \vspace{-6pt}  \\
\includegraphics[scale=0.4,trim=120pt  40pt 180pt 96pt,clip=true]{legend.pdf} \vspace{-0pt} 
\end{tabular}
\begin{tabular}{@{\hspace{-5pt}}c}
energy, method 1: compute solution during forward passes only by sending ``restricted'' messages (sec.~\ref{sec:SRMP})
\end{tabular}
\\
\begin{tabular}{@{\hspace{0pt}}c}\vspace{-12pt} \\
\includegraphics[scale=0.5,trim= 25pt 136pt  39pt 0pt,clip=true]{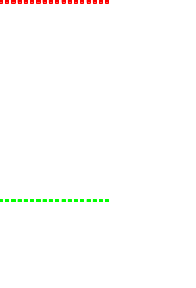} 
\end{tabular}
\begin{tabular}{@{\hspace{-5pt}}c}
\;\!energy, method 2: compute solution during both forward and backward passes of SRMP 
\end{tabular}
\\
\begin{tabular}{@{\hspace{0pt}}l}\vspace{-12pt} \\
\includegraphics[scale=0.5,trim= 25pt  40pt  39pt 96pt,clip=true]{legend3.pdf} \vspace{0pt}
\end{tabular}
\begin{tabular}{@{\hspace{-5pt}}c}
\!energy, method 2: solutions computed by the MPLP code \cite{MPLPcode} (only in (a,g))
\end{tabular}
\\


\framebox{\small
\begin{tabular}{@{\hspace{0pt}}c@{\hspace{-22pt}}c@{\hspace{-16pt}}c@{\hspace{-19pt}}}
\hspace{-113pt}{\bf stereo} \vspace{-6pt} \\
\raisebox{5pt}{\includegraphics[width = 0.36 \textwidth]{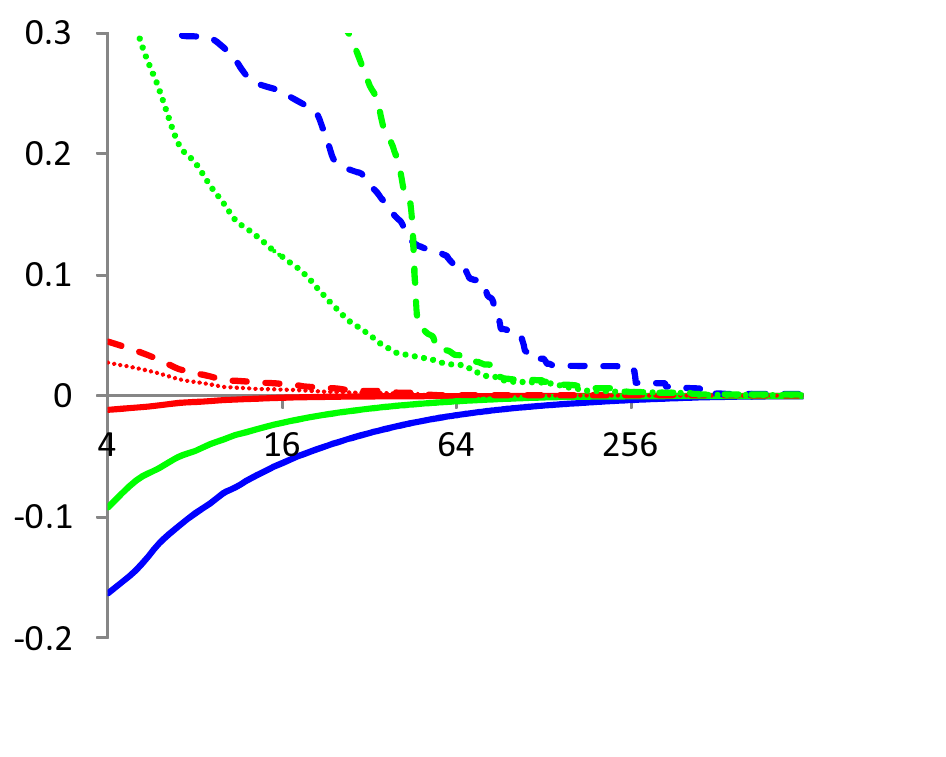}} 
\begin{picture}(0, 0) \put(-53,110){\tiny $2\rightarrow 1$} \end{picture}
\begin{picture}(0, 0) \put(-63,35){\tiny\bf {\color{red}1}/{\color{blue}1.03}/{\color{green}1.43}} \end{picture} &
\includegraphics[width = 0.36 \textwidth]{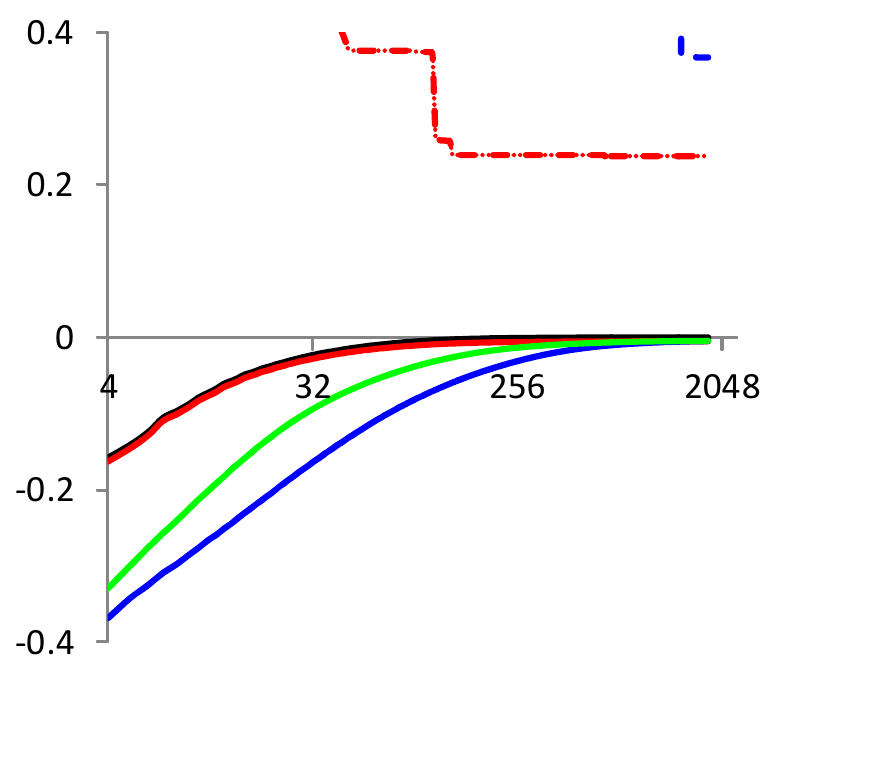} 
\begin{picture}(0, 0) \put(-70,110){\tiny $\{2,\!3\}\rightarrow 1$} \end{picture}
\begin{picture}(0, 0) \put(-63,35){\tiny\bf {\color{red}1}/{\color{blue}1.44}/{\color{green}2.21}} \end{picture} &
\includegraphics[width = 0.36 \textwidth]{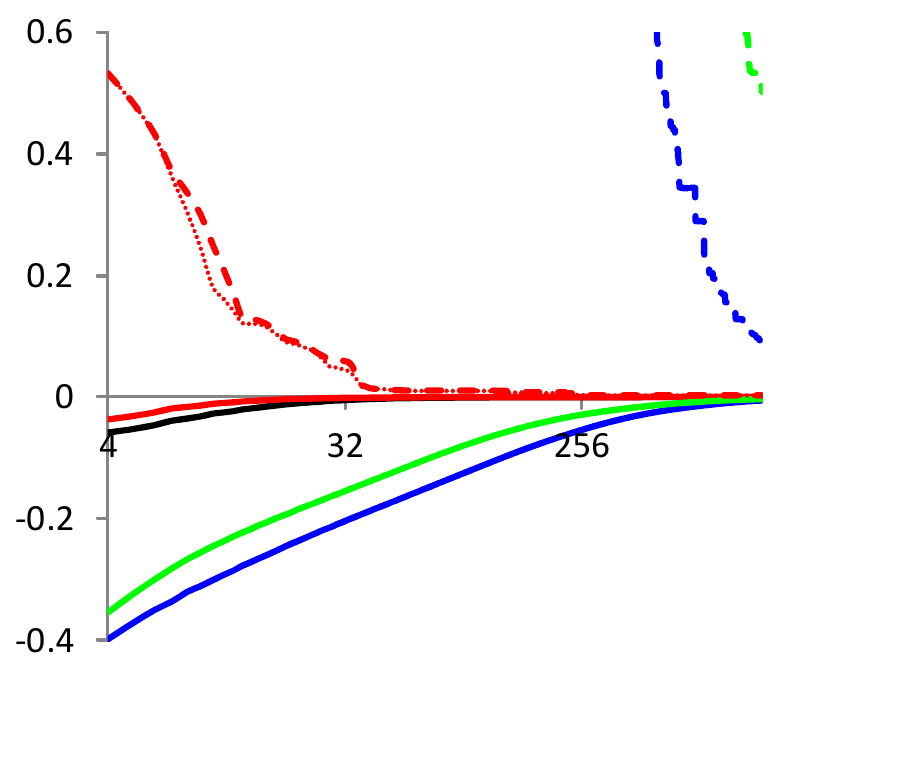} 
\begin{picture}(0, 0) \put(-105,110){\tiny $\{2,\!3\}\rightarrow\{1,\!2\}$} \end{picture}
\begin{picture}(0, 0) \put(-63,35){\tiny\bf {\color{red}1}/{\color{blue}1.25}/{\color{green}1.26}} \end{picture} 
\vspace{-30pt} \\
(a) Potts & (b) second-order, BLP &  (c) second-order 
\end{tabular} 
}
\framebox{\small
\begin{tabular}{@{\hspace{0pt}}c@{\hspace{-23pt}}c@{\hspace{-18pt}}c@{\hspace{-16pt}}}
\hspace{-60pt}{\bf image segmentation} \vspace{-8pt} \\
\includegraphics[width = 0.36 \textwidth]{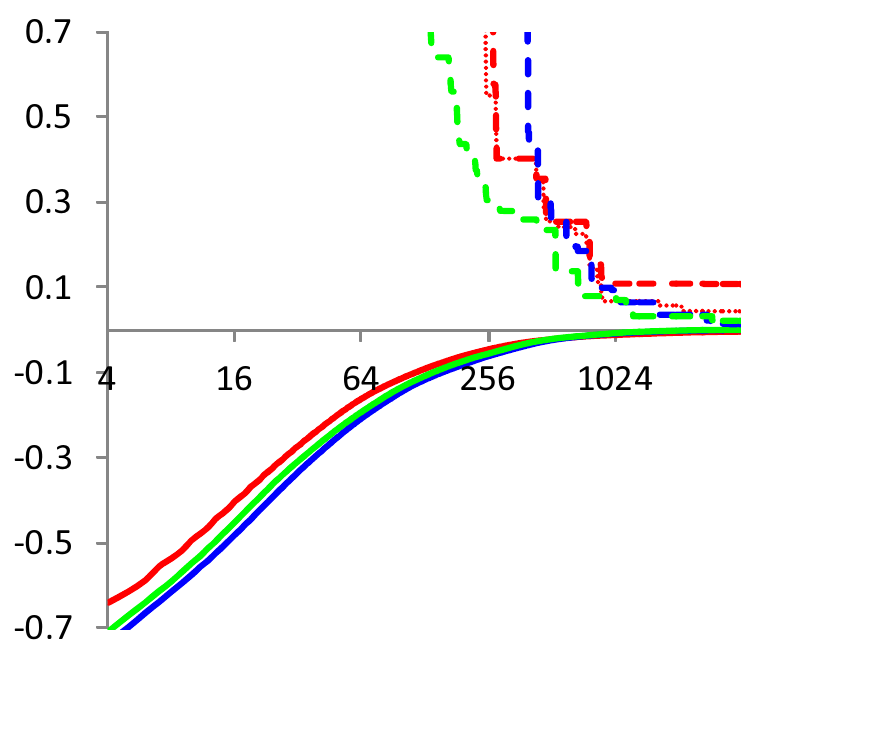}
\begin{picture}(0, 0) \put(-140,100){\tiny $\{2,\!3,\!4\}\!\!\rightarrow\!\!\{1,\!2,\!3\}$} \end{picture}
\begin{picture}(0, 0) \put(-63,35){\tiny\bf {\color{red}1}/{\color{blue}1.67}/{\color{green}1.94}} \end{picture} &
\includegraphics[width = 0.36 \textwidth]{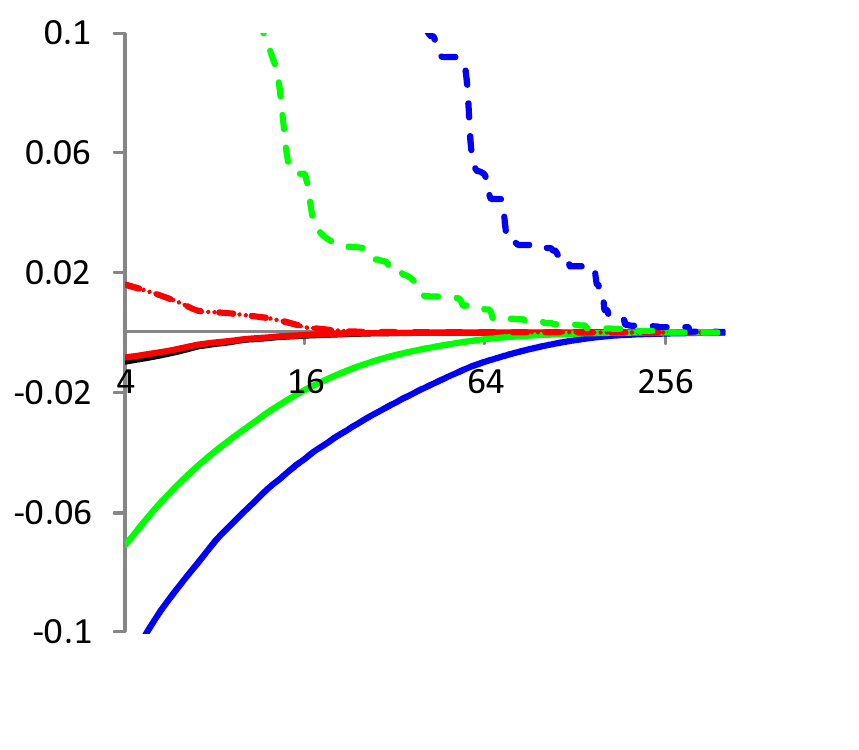} 
\begin{picture}(0, 0) \put(-53,110){\tiny $4\rightarrow 1$} \end{picture}
\begin{picture}(0, 0) \put(-63,35){\tiny\bf {\color{red}1}/{\color{blue}1.60}/{\color{green}3.35}} \end{picture} &
\includegraphics[width = 0.36 \textwidth]{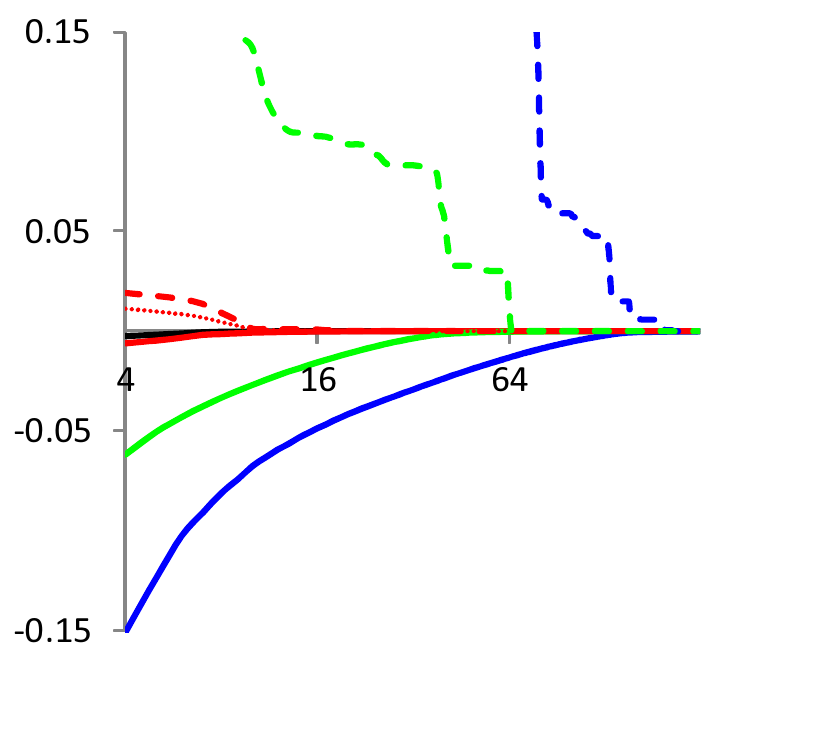} 
\begin{picture}(0, 0) \put(-60,130){\tiny $\{2,\!4\}\!\!\rightarrow\!\!\{1,\!2\}$} \end{picture}
\begin{picture}(0, 0) \put(-63,35){\tiny\bf {\color{red}1}/{\color{blue}1.57}/{\color{green}2.95}} \end{picture} 
\vspace{-30pt} \\
(d) curvature & (e) gen. Potts, BLP &  (f) gen. Potts 
\end{tabular} 
}
\framebox{\small
\begin{tabular}{@{\hspace{0pt}}c@{\hspace{-21pt}}c@{\hspace{-15pt}}c@{\hspace{-12pt}}}
\hspace{-109pt}{\bf proteins} \vspace{-13pt} \\
\raisebox{5pt}{\includegraphics[width = 0.36 \textwidth]{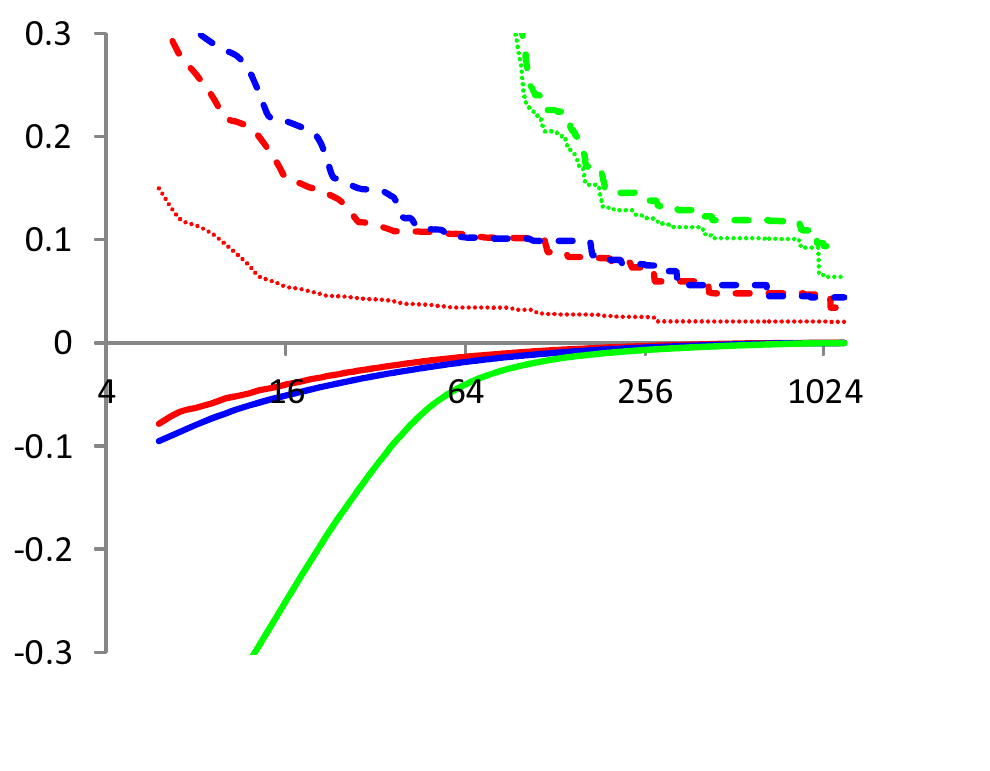}}
\begin{picture}(0, 0) \put(-55,120){\tiny $2\rightarrow 1$} \end{picture}
\begin{picture}(0, 0) \put(-63,35){\tiny\bf {\color{red}1}/{\color{blue}1.14}/{\color{green}0.88}} \end{picture} &
\includegraphics[width = 0.36 \textwidth]{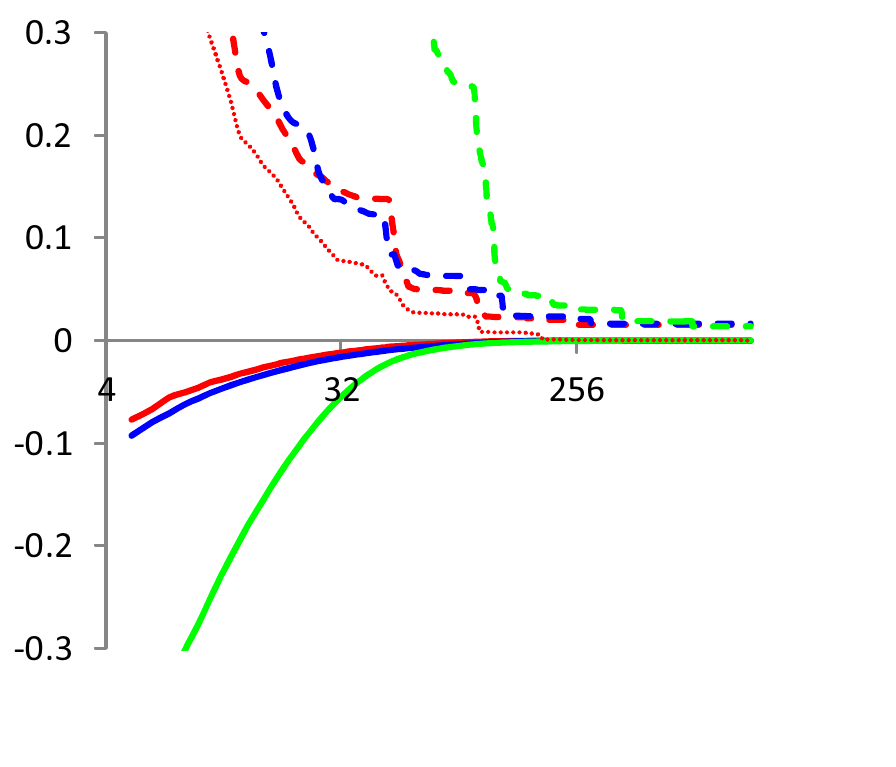} 
\begin{picture}(0, 0) \put(-75,120){\tiny $\{2,\!3\}\rightarrow\{1,\!2\}$} \end{picture}
\begin{picture}(0, 0) \put(-63,35){\tiny\bf {\color{red}1}/{\color{blue}1.25}/{\color{green}1.22}} \end{picture} &
\includegraphics[width = 0.34 \textwidth]{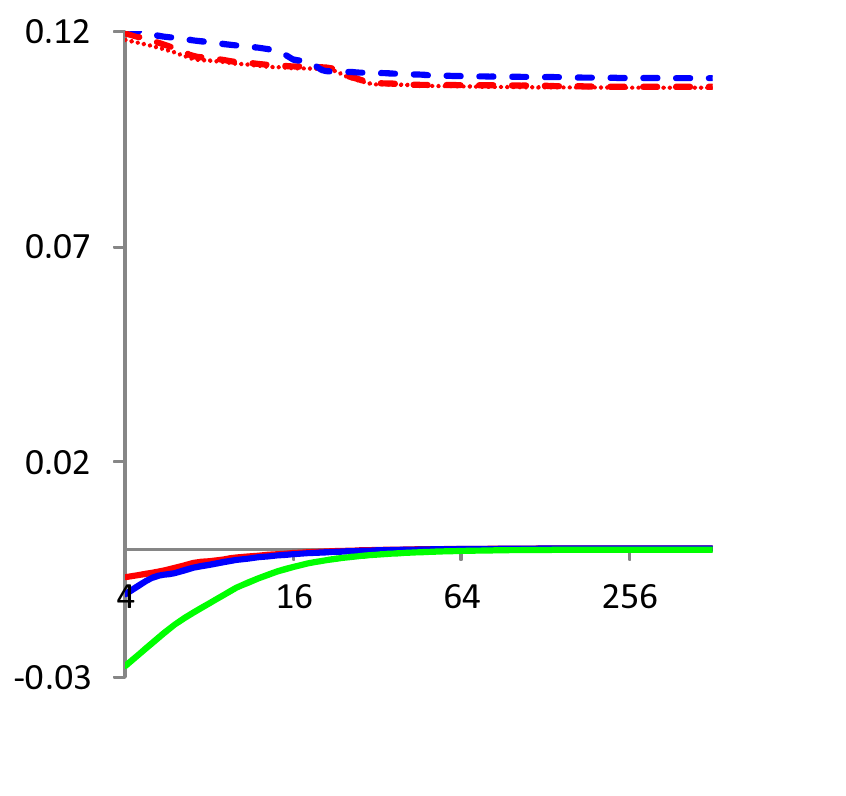} 
\begin{picture}(0, 0) \put(-65,115){\tiny $\{2,\!3\}\rightarrow 1$} \end{picture}
\begin{picture}(0, 0) \put(-63,65){\tiny\bf {\color{red}1}/{\color{blue}1.36}/{\color{green}2.24}} \end{picture} 
\vspace{-30pt} \\
(g) side chains & ~~~~~~~(h) side chains with triplets &  (k) protein interactions
\end{tabular} 
}
\caption{Average lower bound and energy vs. time. {\bf X axis}: SRMP iterations
in the log scale.
Notation {\bf {\color{red}1}/{\color{blue}$a$}/{\color{green}$b$}} means
that 1 iteration of SRMP takes the same time as $a$ iterations of CMP, or $b$ iterations of MPLP.
Note that an SRMP iteration has two passes (forward and backward), while CMP and MPLP iterations have only forward passes.
However in each pass SRMP updates only a subset of messages (half in the case of pairwise models). 
%
{\bf Y axis}: lower bound/energy, normalized so that the initial lower bound (with zero messages) is $-1$,
and the best computed lower bound is $0$. 
%
Line $A\rightarrow B$ gives information about set $J$:\;\; 
 $A=\{|\alpha|\::\:(\alpha,\beta)\in J\}$,\;\; $B=\{|\beta|\::\:(\alpha,\beta)\in J\}$.
\vspace{-10pt}
}
\label{fig:plots}
\end{figure*}

In this section we compare SRMP, CMP (namely, updates from Fig.~\ref{fig:updates}(a) with the uniform distributions $\omega$) and MPLP.
Note that there are many other inference algorithms, see e.g.\ \cite{Kappes:CVPR13} for a recent comprehensive comparison.
Our goal is not to replicate this comparison but instead focus on techniques from the same family 
(namely, coordinate ascent message passing algorithms), since they represent an important branch of
MRF optimization algorithms.

We implemented the three methods in the same framework, trying to put
the same amount of effort into optimizing each technique. For MPLP we used 
equations given in Appendix~\ref{app:MPLPmessages}. On the protein and second-order stereo instances discussed below our implementation 
was about 10 times faster than the code in~\cite{MPLPcode}.\footnote{The code in~\cite{MPLPcode}
appears to use the same message update routines for all factors.
We instead implemented separate routines for different factors (namely Potts, general pairwise and higher-order)
as virtual functions in C++. In addition, we precompute the necessary tables for edges $(\alpha,\beta)\in J$
with $|\alpha|\ge 3$ to allow faster computation.
We made sure that our code produced the same lower bounds as the code in~\cite{MPLPcode}.
}
For all three techniques factors were processed in the order described in section~\ref{sec:SRMP}
(but for CMP and MPLP we used only forward passes).

Unless noted otherwise, graph $(\calF,J)$ was constructed as follows: (i) add to $\calF$ all possible intersections of existing factors;
(ii) add edges $(\alpha,\beta)$ to $J$ such that $\alpha,\beta\in\calF$, $\beta\subset\alpha$, and there is no ``intermediate'' factor $\gamma\in\calF$
with $\beta\subset\gamma\subset\alpha$. For some problems we also experimented with the BLP relaxation (defined in section \ref{sec:background});
although it is weaker in general, message passing operations can potentially be implemented faster for certain factors.


\myparagraph{Instances} We used the data cited in \cite{MPLPdata} 
and a subset of data from~\cite{GTRWS:arXiv12}\footnote{We excluded instances with specialized high-order factors used in~\cite{GTRWS:arXiv12}.
They require customized message passing routines, which we have not implemented. 
As discussed in \cite{GTRWS:arXiv12}, for such energies MPLP has an advantage; SRMP and CMP
could be competitive with MPLP only if the factor allows efficient incremental updates
exploiting the fact that after sending message $\alpha\rightarrow\beta$ only message $m_{\alpha\beta}$ changes
for factor $\alpha$.}.
These are energies of order 2,3 and 4. Note that for energies of order 2
(i.e.\ for pairwise energies) SRMP is equivalent to TRW-S; we included
this case for completeness. The instances are summarized below.

\noindent {\bf (a)} Potts stereo vision. We took 4 instances from~\cite{MPLPdata}; each node has 16 possible labels. \\
{\bf (b,c)} Stereo with a second order smoothness prior~\cite{Woodford:CVPR08}. We used the version described in~\cite{GTRWS:arXiv12}; the energy has unary terms and ternary terms in horizontal and vertical directions.
We ran it on scaled-down stereo pairs ``Tsukuba'' and ``Venus'' (with 8 and 10 labels respectively). \\
{\bf (d)} Constraint-based curvature, as described in~\cite{GTRWS:arXiv12}. Nodes correspond to faces of the dual graph, and have 2 possible labels.
We used ``cameraman'' and ``lena'' images of size 64$\times$64. \\
{\bf (e,f)} Generalized Potts model with 4 labels, as described in~\cite{GTRWS:arXiv12}. The energy has unary terms and 4-th order terms
corresponding to 2$\times$2 patches. We used 3 scaled-down images (``lion'', ``lena'' and ``cameraman''). \\
{\bf (g)} Side chain prediction in proteins. We took 30 instances from  \cite{MPLPdata}. \\ 
{\bf (h)} A tighter relaxation
of energies in (g) obtained by adding zero-cost triplets of nodes to $\calF$.
These triplets were generated by the MPLP code~\cite{MPLPcode} that implements the
techniques in~\cite{Sontag:UAI2008,Sontag:tightening2}.
\footnote{The set of edges $J$ was set in the same way as in the code~\cite{MPLPcode}:
for each triplet $\alpha=\{i,j,k\}$ we add to $J$ edges from $\alpha$ to the factor $\{i,j\}$,
$\{j,k\}$, $\{i,k\}$. If one of these factors (say, $\{i,j\}$) is not present in
the original energy then we did {\bf not} add edges $(\{i,j\},\{i\})$ and $(\{i,j\},\{j\})$.}
\\
{\bf (k)} Protein-protein interactions, with binary labels (8 instances from \cite{PascalInferenceChallenge}).
We also tried reparameterized energies given in~\cite{MPLPdata};
the three methods performed very similarly (not shown here).

\myparagraph{Summary of results} From the plots in Fig.~\ref{fig:plots} we make the following
conclusions. On problems with a highly regular graph structure (a,b,c,e,f)
SRMP clearly outperforms CMP and MPLP. On the protein-related problems (g,h,k)
SRMP and CMP perform similarly, and outperform MPLP. On the remaining problem (d) the three techniques 
are roughly comparable, although SRMP converges to a worse solution. 


%


On a subset of problems (b,c,e,f) we also ran the GTRW-S code from~\cite{GTRWS:arXiv12},
and found that its behaviour is very similar to SRMP - see Fig.~\ref{fig:plots}.\footnote{In the plots (b,c,e,f) we
 assume that one iteration of SRMP takes the same time as one iteration of GTRW-S.
Note that in practice these times were different: SRMP iteration was about 50\% slower than
GTRW-S iteration on (e), but about 9 times faster on (f).}
We do not claim any speed improvement over GTRW-S; instead, the advantage of SRMP is a simpler 
and a more general formulation (as discussed in section~\ref{sec:SRMP}, we believe
that with particular weights SRMP becomes equivalent to GTRW-S).

\section{Conclusions}
We presented a new family of algorithms which includes CMP and TRW-S 
as special cases. The derivation of SRMP is shorter than that of TRW-S; this should facilitate generalizations
to other cases. We developed such a generalization for higher-order graphical models,
but we also envisage other directions. An interesting possibility is to treat
edges in a pairwise graphical model with different weights depending on their ``strengths'';
SRMP provides a natural way to do this. (In $\mbox{TRW-S}$ modifying a weight of 
an individual edge is not easy: these weights depend on probabilities of monotonic
chains that pass through this edge, and changing them would affect other edges as well.)
In certain scenarios it may be desirable to perform updates only in certain parts of
the graphical model (e.g.\ to recompute the result after a small change of the model);
again, SRMP may be more suitable for that. 
\cite{Savchynskyy:UAI12} presented a ``smoothed version of TRW-S'' for pairwise models;
our framework may allow an easy generalization to higher-order graphical models.
We thus hope that our paper will lead to new research directions in the area of approximate MAP-MRF inference.

\appendix

\section{Proof of Proposition~\ref{prop:AMSD}}\label{app:MSDproof}
We can assume w.l.o.g.\ that $\calF=\{\alpha\:|\:(\alpha,\beta)\in \INp{\beta}\}\cup\{\beta\}$
and $J=\IN\beta=\INp{\beta}=\{(\alpha,\beta)\:|\:\alpha\in\calF-\{\beta\}\}$;
removing other factors and edges will not affect the claim.

Let $\theta$ be the output of {\tt AMSD}$(\beta,\IN{\beta},\omega)$.
We need to show that $\Phi(\theta)\ge \Phi(\theta[m])$ for any message vector $m$.

It follows from the description of the AMSD procedure that $\theta$ satisfies the following for any $\bx_\beta$:
\begin{subequations}
\label{eq:prop:AMSD:proof}
\begin{IEEEeqnarray}{rCl}
\hspace{-15pt}
\min_{\bxS_\alpha\sim\bxS_\beta}\theta_\alpha(\bx_\alpha)&=&\omega_{\alpha\beta} \widehat\theta_\beta(\bx_\beta) \qquad \forall (\alpha,\beta)\in J \\
\theta_\beta(\bx_\beta) & = & \omega_\beta \widehat\theta_\beta(\bx_\beta)
\end{IEEEeqnarray}
\end{subequations}
Let $\bx^\ast_\beta$ be a minimizer of $\widehat\theta_\beta(\bx_\beta)$.
From \eqref{eq:prop:AMSD:proof}
we get
\begin{IEEEeqnarray*}{rCl}
\Phi(\theta)&=&\sum_{(\alpha,\beta)\in J}\min_{\bxS_\alpha}\theta_\alpha(\bx_\alpha)
+ \min_{\bxS_\beta}\theta_\beta(\bx_\beta) \\
&=&
\sum_{(\alpha,\beta)\in J} \omega_{\alpha\beta} \widehat\theta_\beta(\bx^\ast_\beta)
+ \omega_\beta\widehat\theta_\beta(\bx^\ast_\beta)
\;=\;\widehat\theta_\beta(\bx^\ast_\beta)
\end{IEEEeqnarray*}
As for the value of $\Phi(\theta[m])$, it is equal to
{\small
\begin{IEEEeqnarray*}{l}
\!\!\!\sum_{(\alpha,\beta)\in J}\!\!\!\min_{\bxS_\alpha}\left[\theta_\alpha(\bx_\alpha)-m(\bx_\beta)\right]
+ \min_{\bxS_\beta}\left[\theta_\beta(\bx_\beta)+\!\!\!\!\sum_{(\alpha,\beta)\in J}\!\!\!m(\bx_\beta)\right] 
 \\
=\!\!\!\!\!\!\sum_{(\alpha,\beta)\in J}\!\!\!\min_{\bxS_\beta}\!\left[\omega_{\alpha\beta}\widehat\theta_\beta(\bx_\beta)\!-\!m(\bx_\beta)\right]
\!\!+\! \min_{\bxS_\beta}\!\!\left[\omega_\beta\widehat\theta_\beta(\bx_\beta)\!+\!\!\!\!\!\sum_{(\alpha,\beta)\in J}\!\!\!\!m(\bx_\beta)\!\right] \\
\le
\!\!\!\sum_{(\alpha,\beta)\in J}\!\!\!\left[\omega_{\alpha\beta}\widehat\theta_\beta(\bx^\ast_\beta)-m(\bx^\ast_\beta)\right]
+ \left[\omega_\beta\widehat\theta_\beta(\bx^\ast_\beta)\!+\!\!\!\!\sum_{(\alpha,\beta)\in J}\!\!\!m(\bx^\ast_\beta)\right] 
 \\
 =\widehat\theta_\beta(\bx^\ast_\beta)
\end{IEEEeqnarray*}
}


\section{Proof of Proposition~\ref{prop:AMPLP}}\label{app:MPLPproof}
We can assume w.l.o.g.\ that $\calF=\{\beta\:|\:(\alpha,\beta)\in \OUTp{\alpha}\}\cup\{\alpha\}$
and $J=\OUT\alpha=\OUTp{\alpha}=\{(\alpha,\beta)\:|\:\beta\in\calF-\{\alpha\}\}$;
removing other factors and edges will not affect the claim.

Let $\theta$ be the output of {\tt AMPLP}$(\alpha,\OUT{\alpha},\rho)$.
We need to show that $\Phi(\theta)\!\ge\! \Phi(\theta[m])$ for any message vector $m$.

Let $\bx^\ast_\alpha$ be a minimizer of $\widehat\theta_\alpha$,
and correspondingly $\bx^\ast_\beta$ be its restriction to factor $\beta\in \calF$.
\begin{proposition}
(a) $\bx^\ast_\alpha$ is a minimizer of $\theta_\alpha(\bx_\alpha)=\widehat\theta_\alpha(\bx_\alpha)-\sum_{(\alpha,\beta)\in J}\rho_{\alpha\beta}\delta_{\alpha\beta}(\bx_\beta)$. \\
(b) $\bx^\ast_\beta$ is a minimizer of $\theta_\beta(\bx_\beta)=\rho_{\alpha\beta}\delta_{\alpha\beta}(\bx_\beta)$ for each $(\alpha,\beta)\in J$. 
\end{proposition}
\begin{proof}
The fact that $\bx^\ast_\beta$ is a minimizer of $\delta_{\alpha\beta}(\bx_\beta)$
 follows directly the definition of vector $\delta_{\alpha\beta}$ in {\tt AMPLP}$(\alpha,\OUT{\alpha},\rho)$.
To show (a), we write the expression as 
\begin{IEEEeqnarray*}{l}
\widehat\theta_\alpha(\bx_\alpha)-\sum_{(\alpha,\beta)\in J}\rho_{\alpha\beta}\delta_{\alpha\beta}(\bx_\beta) \\
\;\;=\rho_\alpha \widehat\theta_\alpha(\bx_\alpha) + \sum_{(\alpha,\beta)\in J}\rho_{\alpha\beta}\left[\widehat\theta_\alpha(\bx_\alpha)-\delta_{\alpha\beta}(\bx_\beta)\right]
\end{IEEEeqnarray*}
and observe that $\bx^\ast_\alpha$ is a minimizer of each term on the RHS; in particular,
$\min\limits_{\bxS_\alpha}\left[\widehat\theta_\alpha(\bx_\alpha)-\delta_{\alpha\beta}(\bx_\beta)\right]=
\widehat\theta_\alpha(\bx^\ast_\alpha)-\delta_{\alpha\beta}(\bx^\ast_\beta)=0$.
\end{proof}

Using the proposition, we can write
\begin{IEEEeqnarray*}{rCl}
\Phi(\theta)&=& \min_{\bxS_\alpha}\theta_\alpha(\bx_\alpha)
+\sum_{(\alpha,\beta)\in J}\min_{\bxS_\beta}\theta_\beta(\bx_\beta)\\
&=&
\theta_\alpha(\bx^\ast_\alpha)
+\sum_{(\alpha,\beta)\in J} \theta_\beta(\bx^\ast_\beta)
\end{IEEEeqnarray*}
As for the value of $\Phi(\theta[m])$, it is equal to
{\small
\begin{IEEEeqnarray*}{l}
 \min_{\bxS_\alpha}\left[\theta_\alpha(\bx_\alpha)-\!\!\!\sum_{(\alpha,\beta)\in J}\!\!\!m(\bx_\beta)\right]
\!+\!\!\!\sum_{(\alpha,\beta)\in J}\!\min_{\bxS_\beta}\left[\theta_\beta(\bx_\beta)+m(\bx_\beta)\right]
 \\
\le
\left[\theta_\alpha(\bx^\ast_\alpha)-\sum_{(\alpha,\beta)\in J}m(\bx^\ast_\beta)\right]
+\sum_{(\alpha,\beta)\in J} \left[\theta_\beta(\bx^\ast_\beta)+m(\bx^\ast_\beta)\right]
 \\
=
\theta_\alpha(\bx^\ast_\alpha)
+\sum_{(\alpha,\beta)\in J} \theta_\beta(\bx^\ast_\beta)
\end{IEEEeqnarray*}
}

\section{Implementation of anisotropic MPLP via messages}\label{app:MPLPmessages}
For simplicity we assume that $\OUTp\alpha=\OUT\alpha$ (which is the case in our implementation).

We keep messages $m_{\alpha\beta}$ for edges $(\alpha,\beta)\in J$
that define current reparameterization $\theta=\bar\theta[m]$ via eq.~\eqref{eq:rep}.
To speed up computations, we also store vector $\theta_\beta$ for all factors $\beta\in\calF$
that have at least one incoming edge $(\alpha,\beta)\in J$. 

The implementation of {\tt AMPLP}$(\alpha,\OUT{\alpha},\rho)$ is given below;
all updates should be done for all labelings $\bx_\alpha,\bx_\beta$ with $\bx_\alpha\sim\bx_\beta$.
In step 1 we essentially set messages $m_{\alpha\beta}$ to zero, and update $\theta_\beta$ accordingly.
(We could have set $m_{\alpha\beta}(\bx_\beta):=0$ explicitly, but this would have no effect.)
In step 2 we move $\theta_\beta$ to factor $\alpha$. Again, we could have set $\theta_\beta(\bx_\beta):=0$
after this step, but this would have no effect.

To avoid the accumulation of numerical errors, once in a while we recompute stored values $\theta_\beta$
from current messages $m$.

\begin{algorithm}[!h]
\begin{algorithmic}[1]
\STATE for each $(\alpha,\beta)\in \OUT{\alpha}$ update
$\theta_\beta(\bx_\beta):=\theta_\beta(\bx_\beta)-m_{\alpha\beta}(\bx_\beta)$,
set $\widehat\theta_\beta(\bx_\beta):=\theta_\beta(\bx_\beta)$\vspace{4pt}
\STATE set~ 
$
\theta_\alpha(\bx_\alpha):= \bar\theta_\alpha(\bx_\alpha) + \!\!\!\!\!\mathlarger\sum\limits_{(\gamma,\alpha)\in \IN\alpha} \!\!\!\!m_{\gamma\alpha}(\bx_\alpha) 
+ \!\!\!\!\!\mathlarger\sum\limits_{(\alpha,\beta)\in \OUT\alpha}\!\!\!\! \theta_\beta(\bx_\beta)
$\!\!\!\!\!\!\!\!\!\!
\STATE for each $(\alpha,\beta)\in \OUT{\alpha}$ update  
\begin{IEEEeqnarray*}{rCl}
\theta_\beta(\bx_\beta) &:=& \rho_{\alpha\beta} \min_{\bxS_\alpha\sim\bxS_\beta} \theta_\alpha(\bx_\alpha)
\\
m_{\alpha\beta}(\bx_\beta)&:=&\theta_\beta(\bx_\beta)-\widehat\theta_\beta(\bx_\beta)
\end{IEEEeqnarray*}
\STATE if we store $\theta_\alpha$ for $\alpha$ (i.e.\ if exists $(\gamma,\alpha)\in J$) then update~
$
\theta_\alpha(\bx_\alpha):=\theta_\alpha(\bx_\alpha)-\mathlarger\sum\limits_{(\alpha,\beta)\in \OUT\alpha} \theta_\beta(\bx_\beta)
$
\end{algorithmic}
\end{algorithm}

\section{Equivalence to TRW-S}\label{app:pairwiseSRMP}
Consider a pairwise model. Experimentally we verified that SRMP
is equivalent to the TRW-S algorithm~\cite{Kolmogorov:TRWS:PAMI06},
assuming that all chains in~\cite{Kolmogorov:TRWS:PAMI06} are assigned the uniform probability
and the weights in SRMP are set via eq.~\eqref{eq:weightsTRWS} and \eqref{eq:weightsTRWS:bw}.
More precisely, the lower bounds produced by the two were identical up to the last digit (eventhough the messages were different).
In this section we describe how this equivalence could be established.

As discussed in section~\ref{sec:SRMP}, the second alternative to implement SRMP is to keep messages $\widehat m$ that define the current reparameterization $\theta$
via the following two-stage procedure: 
(i) compute $\widehat\theta=\bar\theta[\widehat m]$ via eq.~\eqref{eq:rep};
(ii) compute
\begin{eqnarray*}
\theta_\alpha(\bx_\alpha)&\!\!\!\!\!=\!\!\!\!\!&\widehat\theta_\alpha(\bx_\alpha)+\!\!\!\!\!\sum_{(\alpha,\beta)\in \OUT{\alpha}}\!\!\!\!\!\omega_{\alpha\beta}\widehat\theta_\beta(\bx_\beta)\;\;\quad\forall \alpha\in\calF, |\alpha|=2 \qquad\label{eq:reptwo:a'}\\
\theta_\beta(\bx_\beta)&\!\!\!\!\!=\!\!\!\!\!&\omega_{\beta}\widehat\theta_\beta(\bx_\beta)\hspace{68pt}\;\;\quad\forall \beta\in\calF,|\beta|=1\qquad\label{eq:reptwo:b'}
\end{eqnarray*}
where $\omega_{\alpha\beta},\omega_\beta$ are the weights used in the last update for $\beta$.
We also store vectors $\widehat\theta_\beta$ for singleton factors $\beta=\{i\}\in\calS$
and update them as needed. 

Consider the following update for factor $\beta$:


\begin{algorithm}[!h]
\begin{algorithmic}[1]
\STATE for each $\alpha$ with $(\alpha,\beta)\in \IN{\beta}$ update $\widehat m_{\alpha\beta}(\bx_\beta) :=$
\begin{equation*}
\hspace{-8pt} \min_{\bxS_\alpha} \left\{\bar\theta_\alpha(\bx_\alpha)+\!\!\!\sum_{(\alpha,\gamma)\in \OUT\alpha,\gamma\ne\beta }\!
\left[\omega_{\alpha\gamma}\widehat\theta_\gamma(\bx_\gamma)-\widehat m_{\alpha\gamma}(\bx_\gamma)\right]\right\}
\end{equation*}
\STATE update
$\widehat\theta_\beta(\bx_\beta):=\bar\theta_\beta(\bx_\beta)+\sum_{(\alpha,\beta)\in \IN\beta}\widehat m_{\alpha\beta}(\bx_\beta)$
\end{algorithmic}
\end{algorithm}

We claim that this corresponds to procedure {\tt AMSD}$(\beta,\IN{\beta},\omega)$; a proof of this is left to the reader.
This implies that  Algorithm~\ref{alg:SRMP:messages} is equivalent to repeating the following:
(i) run the updates above for factors $\beta\in\calS$ in order $\preceq$ using weights $\omega^+$,
but skip the updates of messages $\widehat m_{\alpha\beta}$ for $(\alpha,\beta)\in \IN{\beta} - \INfw{\beta}$;
(ii) reverse the ordering, 
swap $\INfw{\beta}\!\leftrightarrow\! \INbw{\beta},\omegafw_{\alpha\beta}\!\leftrightarrow\! \omegabw_{\alpha\beta}$. 
\footnote{Skipping updates for $(\alpha,\beta)\in \IN{\beta} - \INfw{\beta}$ is justified as
in Proposition~\ref{prop:SRMPfwbw}. Note, the very first forward pass of the described procedure
is not equivalent to AMSD updates, but is equivalent to the updates in Algorithm~\ref{alg:SRMP:messages};
verifying this claim is again left to the reader.}

It can now be seen that the updates given above are equivalent to those in~\cite{Kolmogorov:TRWS:PAMI06}.
Note, the order of operations is slightly different: in one step of the forward pass we send messages from node $i$ to higher-ordered 
nodes, while in \cite{Kolmogorov:TRWS:PAMI06} messages are sent from lower-ordered nodes to $i$.
However, is can be checked that the result in both cases is the same.

  \section*{Acknowledgments}

I thank Thomas Schoenemann for his help with the experimental section and providing some of the data.
This work was in parts funded by the European Research
Council under the European Unions Seventh Framework Programme (FP7/2007-2013)/ERC grant agreement no 616160.

\bibliography{TRWS}
\bibliographystyle{IEEEtran}

\end{document}